\documentclass[twoside]{article}
\usepackage{fullpage}

\usepackage[utf8]{inputenc} 
\usepackage[T1]{fontenc}    
\usepackage{hyperref}       
\usepackage{url}            
\usepackage{booktabs}       
\usepackage{amsfonts}       
\usepackage{nicefrac}       
\usepackage{microtype}      
\usepackage{algorithm}
\usepackage{algpseudocode}
\usepackage{algorithmicx}
\usepackage{amsmath}
\usepackage{amsthm}
\usepackage{amssymb}
\usepackage{xspace}
\usepackage{mathrsfs,xcolor}
\usepackage[normalem]{ulem}
\usepackage{subcaption}
\usepackage{graphicx}
\usepackage{natbib}
\newcommand{\ssp}{\vspace{0mm}}
\newcommand{\msp}{\vspace{0mm}}
\newcommand{\lsp}{\vspace{0mm}}

\newcommand{\ben}{\begin{enumerate*}}
\newcommand{\een}{\end{enumerate*}}
\newcommand{\beq}{\begin{eqnarray}}
\newcommand{\eeq}{\end{eqnarray}}
\newcommand{\bit}{\begin{itemize*}}
\newcommand{\eit}{\end{itemize*}}

\newcommand{\hide}[1]{}

\newtheorem{theorem}{Theorem} 
\newtheorem{definition}{Definition} 

\newcommand{\indep}[1]{\perp}

\DeclareMathOperator*{\argmax}{arg\,max}

\newcommand{\appropto}{\mathrel{\vcenter{
  \offinterlineskip\halign{\hfil$##$\cr
    \propto\cr\noalign{\kern2pt}\sim\cr\noalign{\kern-2pt}}}}}

\renewcommand{\eqref}[1]{Equation~(\ref{#1})}

\newtheorem{corollary}[theorem]{Corollary}
\newtheorem{prop}[theorem]{Proposition}
\newtheorem{lemma}[theorem]{Lemma}
\newtheorem{claim}[theorem]{Claim}
\newtheorem{assumption}{Assumption}
\newcommand{\truthfunc}{f}
\newcommand{\compfunc}{f_c}
\newcommand{\funcdiff}{\zeta}
\newcommand{\costdirect}{\lambda_l}
\newcommand{\costcomp}{\lambda_c}
\newcommand{\threscomp}{\gamma}
\newcommand{\kernel}{\kappa}
\newcommand{\truthfuncNormBound}{B}
\newcommand{\featurespace}{\mathcal{X}}
\newcommand{\feature}{x}

\newcommand{\directval}{y}

\newcommand{\dataset}{\mathcal{D}}

\newcommand{\infogain}[2]{\Phi_{#1}(#2)}
\newcommand{\logO}{\tilde{O}}
\newcommand{\shiftAlgoName}{COMP-GP-UCB\xspace}
\newcommand{\labmu}{\mu^{(l)}}
\newcommand{\labsig}{\sigma^{(l)}}
\newcommand{\labelerrbound}{\eta}
\algrenewcommand\algorithmicrequire{\textbf{Input:}}
\algrenewcommand\algorithmicensure{\textbf{Output:}}
\newcommand{\compprobfunc}{f_r}
\newcommand{\linkfunc}{\sigma}
\newcommand{\liplink}{L_2}
\newcommand{\liplinklower}{L_1}

\newcommand{\baseset}{\mathcal{H}}

\newcommand{\basec}{\overline{\mathcal{H}^{\threscomp}}}
\newcommand{\baset}{\baseset^{\threscomp}}
\newcommand{\nqueryr}{T^r}
\newcommand{\nqueryl}{T^l}
\newcommand{\baser}{\overline{\mathcal{H}^{\threscomp}}}

\newcommand{\nlabelq}{N_l}

\newcommand{\iterdiff}{k}
\newcommand{\changepoint}{\tilde{N}}
\newcommand{\budget}{\Lambda}
\newcommand{\diffupper}{\funcdiff_{\max}}
\newcommand{\startupcomp}{N_0}
\newcommand{\startupbudget}{\Lambda_0}
\newcommand{\equivdiff}{\bar{\funcdiff}}
\newcommand{\basetp}{\hat{H}^{\threscomp}}

\newcommand{\conf}{\phi}
\newcommand{\confr}{\phi^{(r)}}

\newcommand{\rmu}{\mu^{(r)}}
\newcommand{\rsigma}{\sigma^{(r)}}
\newcommand{\boundr}{\beta^{(r)}}
\newcommand{\boundlab}{\beta}
\newcommand{\querytype}{q}
\newcommand{\estbestfr}{\hat{f}_r}

\title{Zeroth Order Non-convex optimization with Dueling-Choice Bandits}
\author{
	Yichong Xu\thanks{Carnegie Mellon University, Pittsburgh, USA. Email: yichongx@cs.cmu.edu} \and
	Aparna Joshi\thanks{Carnegie Mellon University, Pittsburgh, USA. Email: aparnaj@cs.cmu.edu} \and
	Aarti Singh\thanks{Carnegie Mellon University, Pittsburgh, USA. Email: aarti@cs.cmu.edu} \and
	Artur Dubrawski\thanks{Carnegie Mellon University, Pittsburgh, USA. Email: awd@cs.cmu.edu} 
}
\date{}

\begin{document}

%

%

\maketitle

\begin{abstract}
We consider a novel setting of zeroth order non-convex optimization, where in addition to querying the function value at a given point, we can also duel two points and get the point with the larger function value. We refer to this setting as optimization with dueling-choice bandits since both direct queries and duels are available for optimization. We give the COMP-GP-UCB algorithm based on GP-UCB \citep{srinivas2009gaussian}, where instead of directly querying the point with the maximum Upper Confidence Bound (UCB), we perform a constrained optimization and use comparisons to filter out suboptimal points. COMP-GP-UCB comes with theoretical guarantee of $O(\frac{\Phi}{\sqrt{T}})$ on simple regret where $T$ is the number of direct queries and $\Phi$ is an improved information gain corresponding to a comparison based constraint set that restricts the search space for the optimum. In contrast, in the direct query only setting, $\Phi$ depends on the entire domain. Finally, we present experimental results to show the efficacy of our algorithm.
\end{abstract}

	\msp
\section{Introduction}
\ssp
Zeroth order non-convex optimization, also variously known as continuous multi-armed bandit or black-box optimization, is an important problem that naturally appears in various domains like dynamic pricing \citep{besbes2009dynamic}, reinforcement learning \citep{smart2000practical} and material science\citep{xue2016accelerated}. With an unknown black-box function $\truthfunc:\featurespace\rightarrow \mathbb{R}$, zeroth order optimization aims to find the optimal point of the function with as few queries to (a noisy version of) $\truthfunc(x)$ as possible, with no gradient information directly available. Although zeroth order {\em convex} optimization is generally efficient \citep{jamieson2012query}, optimizing a {\em non-convex} $\truthfunc$ under smoothness constraints requires the same effort as estimating $\truthfunc$ almost everywhere, and usually leads to a query complexity exponential in $d$, where $d$ is the feature space dimensionality\citep{carmon2017convex,chen1988lower,flaxman2005online,ge2015escaping,wang2018optimization}. The prohibitive cost for non-convex optimization has motivated research on suitable assumptions, such as linear bandits \citep{rusmevichientong2010linearly}, convex approximations \citep{wang2018optimization}, and optimization based on level sets \citep{malherbe2016ranking,wang2018optimization}. We propose a complementary approach, where in addition to direct queries to $\truthfunc$, one can also \emph{compare} two points in feature space, and obtain the point with a larger $\truthfunc$ value. Inspired by dueling bandits in the bandit domain \citep{yue2012k}, we call our setting non-convex optimization with dueling-choice bandits; we note that different than dueling bandits, here direct queries and comparisons are both available for optimizing $\truthfunc$, and therefore duels are available as an additional choice.

In many applications, comparisons can be available at a cheaper price than direct queries. For example in preference elicitation, the user can give scores on recommended items, as well as (more easily) compare two items to choose the preferred one. Similarly, for hyperparameter search of information retrieval (IR) models, direct queries typically involve collecting relevance scores from paid workers, whereas comparisons can be obtained by interleaving the ranking of two different models, and observing user click on the retrieval results \citep{radlinski2008does}. Such comparisons usually come with less cost in both time and money. As another example, in material synthesis, we aim to optimize the desired properties of materials by controlling the input parameters (temperature, pressure, etc.) \citep{faber2016machine,xue2016accelerated}. While material synthesis is expensive, comparisons can be carried out by asking material scientists. 

\textbf{Related Work.} There is a vast literature on zeroth order non-convex optimization \citep{Bull:2011:CRE:1953048.2078198,carmon2017convex,chen1988lower,flaxman2005online,ge2015escaping,hazan2017hyperparameter,wang2018optimization}. We build our work on GP-UCB\citep{srinivas2009gaussian}, a method for optimizing unknown functions under the Gaussian process (GP) assumption by optimizing the Upper Confidence Bound (UCB). Closest to our setting is a line of recent research on multi-fidelity GP optimization \citep{kandasamy2016multi,Kandasamy:2017:MBO:3305381.3305567,pmlr-v80-sen18a}, which assumes that we can query the target functions at multiple fidelities of different costs and precisions. We detail the relation and difference of our setting with multi-fidelity optimization in Section \ref{sec:with_mfgpucb}. To briefly describe it, our setting is harder since we cannot directly query the function on which comparisons are based. Moreover, the multi-fidelity assumptions such as fidelities being close in sup-norm do not hold for our setting since any constant shift of the comparison function yields the same comparisons. We instead consider an {\bf active transfer learning} setting where information from a function that can be learned using comparisons is transferred actively to optimize the target function (refer to Section \ref{sec:setup} for details). 

Optimization with comparisons has been studied under the framework of derivative-free optimization \citep{jamieson2012query,kumagai2017regret} and continuous dueling bandits \citep{ailon2014reducing,sui2017multi}. Kumagai \citep{kumagai2017regret} obtains optimal regret rates for a convex $\truthfunc$.
However to the best of our knowledge, no previous work has theoretical guarantees on optimizing a non-convex $\truthfunc$. Also, these results cannot be applied when the comparisons are biased (i.e., a Condorcet winner on comparisons might not be the best point for direct queries).

Finally, there is another line of research that combines direct queries and comparisons for classification or regression problems \citep{kane2017active,xu2017noise,xu2018nonparametric}. Our methods differ from theirs because we focus on the optimization setting, and only care about points near the optimum. These methods make direct queries across the whole feature space to learn the underlying function well, which is unnecessary for optimization.

\textbf{Our Contributions.} We develop and evaluate a new algorithm for non-convex optimization with dueling choices, which we refer to as Comparison-based GP-UCB (\shiftAlgoName). Our theoretical and experimental results show the strengths of our algorithm.

\begin{itemize}
	\item When we can obtain comparisons based on the target function $\truthfunc$, we show that comparisons can be as powerful as direct queries: \shiftAlgoName can achieve the \emph{same} rate of convergence as its label-only counterparts, while using only comparisons and no direct queries. This solves the open problem raised in \cite{sui2017multi}, to develop continuous dueling bandit algorithms with no-regret guarantees.
	\item Next, we assume that comparisons are based on a misspecified function $\compfunc$, where $\compfunc$ approximates $\truthfunc$. \shiftAlgoName in this case uses comparisons to optimize a function $f_r$ which has the same optimizer as $f_c$, and then use direct queries to search in a smaller region for the optimum of the target function. The regret rate of \shiftAlgoName is then better than the label-only counterparts, and it depends on the difference between $\compfunc$ and $\truthfunc$: the better the approximation, the lower the regret we can get from \shiftAlgoName. 
	We further demonstrate a version of the algorithm that adapts to this difference. Our algorithm also extends multi-fidelity GP optimization to the setting where information is transferred actively from a lower fidelity to a higher fidelity while only assuming that the optimizer of the lower fidelity (source function) 
	is within a constant distance of the optimizer of the higher fidelity (target function), instead of the fidelities being close everywhere. 
	\item In our experiments, we test \shiftAlgoName on multi-fidelity functions from previous literature and show that it outperforms label-only algorithms and existing multi-fidelity algorithms when comparisons are cheaper than direct labels.
\end{itemize}

%

\msp
\section{Background and Problem Setup\label{sec:setup}}
\ssp
	
	
	We aim to maximize a function $\truthfunc:\featurespace\rightarrow \mathbb{R}$, where $\featurespace \subseteq [0,r]^d$ is the feature space. In each iteration $t$ of optimization, we can query (expensive) direct queries to $\truthfunc$ at a chosen point $x_t$, and obtain $\directval=f(x_t)+\varepsilon$, $\varepsilon\in [-\labelerrbound,\labelerrbound]$ and $E[\varepsilon]=0$, with $\labelerrbound>0$ a known constant\footnote{Our methods can also be extended to the setting where $\varepsilon$ follows a sub-Gaussian distribution with parameter $\labelerrbound$. We assume a bounded $\varepsilon$ for simplicity here.}.
	
	\textbf{Comparison Probabilities.} In addition to traditional direct queries $\directval$, we can obtain (cheap) comparisons for a pair of points $(x_t,x_t')\in \featurespace \times \featurespace$. We assume that comparisons are based on a function $\compfunc$ which can be potentially different from $\truthfunc$ (as described later in this section). A common assumption in the literature is to use a link function to assume a distribution of the comparisons, i.e., we assume $\Pr[x\succ x']=\linkfunc(\compfunc(\feature)-\compfunc(\feature'))$ for some function $\linkfunc$. Common link functions include logistic function (BTL model\citep{bradley1952rank}), or Gaussian cdf (Thurstone model \citep{thurstone1927law}).

	\textbf{Connecting comparisons and direct queries.} 
	To make comparisons helpful for optimization, 
	we also require that $\compfunc$ is a good approximation of $\truthfunc$. Here we differentiate between two settings:
	\lsp
	\begin{itemize}
		\item Dueling-Choice Bandits with unbiased comparisons: We assume comparison comes from the same function as the target function, i.e., $\compfunc=\truthfunc$ or, more generally, that $\compfunc$ and $\truthfunc$ have the same optimizer ($\zeta = 0$ as described below). This may be the case when comparison and direct queries come from the same agent, such as the preference elicitation example in the introduction.
		
		\item Dueling-Choice Bandits with misspecified comparisons: We assume $\compfunc\approx \truthfunc$. In many cases, comparisons are from a different source (e.g. experts) than direct queries (e.g. experiments), and this can result in a biased $\compfunc$. To this end, we assume a bounded difference near the optimum:
		\newcommand{\opteffrange}{\iota}
		\begin{assumption}\label{asm:funcdiff}
			Let $\truthfunc^*=\max_x \truthfunc(x)$ and $\compfunc^*=\max_x \compfunc(x)$. There exists a constants $\funcdiff$ such that for any point $x\in \featurespace$
			 we have
			$|(\compfunc^*-\compfunc(x)) - (\truthfunc^*-\truthfunc(x))|\leq \funcdiff$. 
		\end{assumption}
		In words, when we get $\varepsilon$-close to the maximum of $\truthfunc$, we are at least $(\varepsilon+\funcdiff)$-close to the maximum of $\compfunc$, and vice versa. Under this assumption, we would require both comparison and direct queries if we want to achieve optimization error smaller than $\funcdiff$. 
		
		We note that our results can be generalized to the case where Assumption \ref{asm:funcdiff} only holds for $x\in\{x: \truthfunc^*-\truthfunc(x)\leq \tau \}$ for some fixed constant $\tau$.
	\end{itemize}
	\textbf{Smoothness Assumptions.} We assume that the target function $\truthfunc$ lies in a reproducing kernel Hilbert space (RKHS) $\mathcal{H}_\kernel$ induced by kernel $\kernel$, and that the RKHS norm of $\truthfunc$ is bounded: $\|\truthfunc\|_{\kernel}\leq B$ for a known constant $\truthfuncNormBound$.
	This assumption is also analyzed in \citep{chowdhury2017kernelized,srinivas2009gaussian} for traditional bandits. We note that every function $\truthfunc\in \mathcal{H}_\kernel$ has a finite kernel norm. When $\kernel$ is the linear kernel, $\|\truthfunc\|_{\kernel}\leq B$ induces that $\truthfunc$ is $\truthfuncNormBound$-Lipschitz.

	\newcommand{\upperquery}{\overline{n}_{\Lambda}}
	\newcommand{\lowerquery}{\underline{n}_{\Lambda}}
	\newcommand{\comptype}{\text{comp}}
	\newcommand{\labeltype}{\text{label}}
	\newcommand{\simpleregret}{S}
	
	\textbf{Budgets and Regrets. } We analyze the problem of optimizing $\truthfunc$ under a given cost budget $\budget$. Suppose a direct query costs $\costdirect$ units of some resource and a comparison costs $\costcomp<\costdirect$. 
	Also, let $\upperquery=\lceil \frac{\Lambda}{\costcomp}\rceil$ be the upper bound on number of queries when we use all the budget on comparisons, and $\lowerquery=\lfloor \frac{\Lambda}{\costdirect}\rfloor$ be the corresponding lower bound when we only use direct queries. 
	Also let $q_t=\labeltype$ if we make direct queries at iteration $t$, and $q_t=\comptype$ otherwise.
	We analyze the simple regret under budget $\budget$, defined as follows:
	\begin{align}
&	\simpleregret(\Lambda)=\min_tr_t\label{eqn:regret_def}\\
	&=\min_t\begin{cases}
	\displaystyle  \truthfunc^*-\truthfunc(x_t) & \text{ if } q_t=\labeltype,\\
	\min\{\truthfunc^*-\truthfunc(x_t),\truthfunc^*-\truthfunc(x'_t)\}, & \text{ if } q_t=\comptype.
	\end{cases}\nonumber
	\end{align}
	In words, we calculate the minimum regret achieved by either comparison or direct queries. We compute simple regret over all direct queries; for comparisons, we adopt the notion of weak regret employed in \citep{yue2012k}. Here we choose simple regret because our target is to optimize function $\truthfunc$, and cumulative regret is typically not relevant for our setting. Our method can also be easily extended to the optimizer error setting, where the algorithm gives an estimation of the optimum when it ends. In analyzing the regret rates, we use $O(\cdot)$ to ignore constants, and $\logO(\cdot)$ to ignore log terms in the regret bounds.



\msp
\section{Algorithm and Analysis}

\ssp
We describe our \shiftAlgoName algorithm in this section. We first present the Gaussian process framework on which our algorithm is based in Section \ref{sec:gp}. Then we present the algorithm in Section \ref{sec:algo_knowndiff}, under the assumption that $\funcdiff$ is known. This includes the unbiased comparison case, where $\funcdiff=0$. We present our theoretical analysis for \shiftAlgoName in Section \ref{sec:theory_analysis}. Finally, we give an extension to adapt to unknown $\funcdiff>0$ in Section \ref{sec:adapt_noise}.
\newcommand{\GP}{\mathcal{GP}}
\ssp
\subsection{The Gaussian Process Back End\label{sec:gp}}
\ssp
We base our methods on Gaussian Process, with kernel function $\kernel$. If $\truthfunc$ was sampled from the Gaussian process $\mathcal{GP}(0,\kernel)$, and the direct queries were coming from $\truthfunc$ plus a Gaussian noise, i.e., $\dataset=\{(\feature_i,\directval_i)\}_{i=1}^t$ with $\directval_i=f(x_i)+\varepsilon$, $\varepsilon\sim \mathcal{N}(0,\labelerrbound^2)$, then
the posterior distribution at $\truthfunc(\feature)|\dataset$ would be a Gaussian $\mathcal{N}(\mu_t(\feature),\sigma_t(\feature))$ with
\begin{align}
\mu_t(\feature)&=k^T(K+\labelerrbound^2I_t)^{-1}Y,\label{eqn:bayes_update}\\ \sigma_t(\feature)&=\kernel(\feature,\feature)-k^T(K+\labelerrbound^2I_t)^{-1}k. \nonumber
\end{align}
Here $Y=(\directval_1,...,\directval_t)^T$, $k=(\kernel(\feature, \feature_1),...,\kernel(\feature, \feature_t))^T$, and matrix $K\in \mathbb{R}^{t\times t}$ is given by $K_{ij}=\kernel(\feature_i, \feature_j)$, and $I_t$ is the $t\times t$ identity matrix.

\textbf{Remark.} We note that the Gaussian noise and prior is only assumed to derive updates to the mean and variance in the algorithm, and we do not assume the actual feedbacks follow a Gaussian model, nor the functions are sampled from the Gaussian process. 
We only assume that $\truthfunc$ have bounded norm in $\mathcal{H}_\kernel$ and that $\varepsilon$ is bounded in $[-\labelerrbound,\labelerrbound]$, as stated in Section \ref{sec:setup}.
This is the same as the \emph{agnostic} setting in GP-UCB \citep{chowdhury2017kernelized,srinivas2009gaussian}.

\textbf{The Maximum Information Gain.} As in previous works on GP \citep{chowdhury2017kernelized, kandasamy2016multi}, our results will depend on the \emph{maximum information gain} \citep{srinivas2009gaussian} between function measurements and the function values, defined as below:
\begin{definition}
	Suppose $A\subseteq \featurespace$ is a subset of feature space, and $\tilde{A}=\{\feature_1,...,\feature_n \}\subseteq A$ is a finite subset of $A$. Then the maximum information gain on $A$ with $n$ evaluations is defined as
	\(\infogain{n}{A}=\max_{\tilde{A}\subseteq A, |\tilde{A}|=n}I(\truthfunc_{\tilde{A}}+\varepsilon_{\tilde{A}};\truthfunc_{\tilde{A}}), \)
	where $\truthfunc_{\tilde{A}}=[\truthfunc(\feature)]_{\feature\in \tilde{A}}$, $\varepsilon_{\tilde{A}}\sim \mathcal{N}(0,\eta^2 I)$, and $I(X,Y)=H(X)-H(X|Y)$ is the mutual information.
\end{definition}

When $\featurespace\subseteq \mathbb{R}^d$ is compact and convex, \cite{srinivas2009gaussian} shows that i) for linear kernel $\kernel$, $\infogain{n}{\featurespace}=O(d\log n)$; ii) for squared exponential (SE, or RBF) kernel, $\infogain{n}{\featurespace}=O((\log n)^{d+1})$; iii) For Mat\'{e}rn kernels $\kernel(\feature, \feature')=\frac{2^{1-\nu}}{\Gamma(\nu)}(\frac{\sqrt{2\nu}z}{\rho})^\nu B_{\nu}(\frac{\sqrt{2\nu}z}{\rho})$, we have $\infogain{n}{\featurespace}=O\left(n^{\frac{d(d+1)}{2\nu+d(d+1)}}\log n\right)$.

\textbf{Review of GP-UCB and IGP-UCB.}	
Previous sequential optimization has adopted the upper confidence bound (UCB) principle, where we maintain a high-confidence upper bound $\conf:\featurespace\rightarrow \mathbb{R}$ for all $x\in \featurespace$, such that $\truthfunc(x)\leq \conf(x)$ with high probability. Our algorithm builds on UCB algorithms for GP, namely GP-UCB \citep{srinivas2009gaussian} and IGP-UCB \citep{chowdhury2017kernelized} (the latter is an improvement of the former).
In time step $t$ of optimization, IGP-UCB queries the point that maximizes the confidence upper bound in the form $\labmu_{t-1}(x)+\boundlab_t\labsig_{t-1}(x)$, where $\labmu_{t-1},(\labsig_{t-1})^2$ are the posterior mean and variance function of the GP from step $t-1$, and $\boundlab_t$ is a multiplier that increases with $t$. We describe these algorithms in detail in Appendix.

\subsection{The Borda Function $\compprobfunc$}
A straightforward way to incorporate comparisons into optimization is to use them to compute a GP posterior of either $\truthfunc$ or $\compfunc$. However, we will face several difficulties. Firstly, the posterior based on comparisons cannot be analytically computed. Also, we cannot compute the joint posterior based on both direct queries and comparisons, since $\truthfunc$ and $\compfunc$ can be different. Lastly, comparisons might not be truthful and can be inconsistent; i.e., human might give contradicting comparisons like $x_1\succ x_2\succ x_3\succ x_1$ \citep{NIPS2015_6023}.

We instead consider a different function directly related to $\compfunc$, defined as $\compprobfunc(x)=\Pr[x\succ X]$, where $X$ is randomly chosen from $\featurespace$. In words, $\compprobfunc(x)$ is the probability that $x$ beats a random point $X\in \featurespace$. We refer to $\compprobfunc$ as the Borda function, inspired by Borda scores in the dueling bandits literature \citep{heckel2016active,NIPS2015_6023}. An advantage of using $\compprobfunc$ is that we can obtain unbiased estimates of $\compprobfunc(x)$ by comparing $x$ to a random point in $X\in \featurespace$. 

It is easy to see that $\compprobfunc$ should have the same optimizer as $\compfunc$. 
We make the following assumption to ensure usefulness of comparisons:
\begin{assumption}\label{asm:bound_diff}
	Let $\compprobfunc^*=\max_x \compprobfunc(x)$ and $\compfunc^*=\max_x \compfunc(x)$. There exists constants $\liplinklower, \liplink$ such that for every $x\in \featurespace$ we have
	\(\frac{1}{\liplinklower}(\compfunc^*-\compfunc(x))\leq \compprobfunc^*-\compprobfunc(x)\leq \liplink(\compfunc^*-\compfunc(x)). \)
\end{assumption}
In other words, difference in $\compfunc$ will cause a difference of similar scale in $\compprobfunc$. This requires that the comparisons induces a Borda function $\compprobfunc$ such that $\compprobfunc$ is close to $\compfunc$ at its optimum, and that $\compprobfunc$ and $\compfunc$ has the same optimizer. 
We note that this is a quite weak assumption, as we do not restrict the result of comparing individual points $\feature,\feature'$ to comply with $\compfunc(x)-\compfunc(x')$, i.e., comparisons do not need to be consistent.
We can show that Assumption \ref{asm:bound_diff} holds under the link function setting, when $\linkfunc$ is Lipschitz continuous:
\begin{prop}\label{prop:liplink}
	Suppose comparisons follows a link function $\linkfunc$ with a Lipschitz constant between $[1/\liplinklower,\liplink]$, i.e., $\frac{|\linkfunc(x)-\linkfunc(y)|}{|x-y|}\in [\frac{1}{\liplinklower},\liplink]$, $\forall x,y\in \mathbb{R}$, then Assumption \ref{asm:bound_diff} holds.
\end{prop}
We comment that common link functions such as BTL \citep{bradley1952rank} and Thurstone \citep{thurstone1927law} all have bounded Lipschitz functions if $\compfunc$ is bounded.

Lastly, we note that \cite{ailon2014reducing} also compare $x$ to a random point $X$, and use the feedback to update the function estimates. However, their method relies on a linear link function $\sigma(x)=\frac{1+x}{2}$ and cannot be applied for BTL or Thurstone models.

\begin{algorithm}[htb]
	\caption{\shiftAlgoName}
	\begin{algorithmic}[1]
		\Require Comparison bias $\funcdiff$, comparison exploration threshold $\threscomp$, confidence $\delta$
		\State Set $D_0^r=D_0^l=\emptyset$, $(\rmu_0,\rsigma_0)=(\labmu_0,\labsig_0)=(0,\kernel^{1/2}),t\leftarrow 0$
		\Repeat\label{step:firstphase}
		\State Compute $x_t=\argmax_{\feature\in \featurespace} \rmu_{t-1}(\feature)+\boundr_{t}\rsigma_{t-1}(\feature)$\label{step:firstphasext}
		\State \textsc{Query}($x_t,\comptype$)
		\Until{$\boundr_{t}\rsigma_{t-1}(\feature_t)\leq \threscomp$ or budget exhausted}\label{step:firstphaseends}
		\State Let $\estbestfr=\rmu_{t-1}(\feature_t)-\boundr_{t}\rsigma_{t-1}(\feature_t)$ \label{step:lcb}
		\While{Budget not exhausted}\label{step:secondphase}
		\State Let $\confr_t(x)=\rmu_{t-1}(x)+\boundr_t\rsigma_{t-1}(x)-\estbestfr+\liplink\funcdiff$
		\State Compute $x_t=\argmax_{x\in \featurespace: \confr_t(x)\geq 0} \labmu_{t-1}(x)+\boundlab_t\labsig_{t-1}(x)$
		\If{$\boundr_t(x_t)\rsigma_{t-1}(x_t)\geq \gamma$} \label{step:checkconf_p2} \textsc{Query}$(x_t,\text{comp})$
		\Else \; \textsc{Query}$(x_t,\text{label})$
		\EndIf
		\State $t\leftarrow t+1$
		\EndWhile\label{step:secondphaseends}
		\Statex\hrulefill
		\Procedure{Query}{query point $x_t$, query type $\querytype_t$}
		\If{$\querytype_t=\text{comp}$}
		\State Sample $\feature'$ randomly from $\featurespace$ and query to compare $(\feature_t,\feature')$, obtain $z_t$
		\State Update $D_t^c\leftarrow D_{t-1}^c\cup \{(x_t,z_t)\}$, $D_t^l\rightarrow D_{t-1}^l$
		\State Perform Bayesian update for $\rmu_{t}, \rsigma_t$ based on $D_t^c$ with $y_t=z_t$ following (\ref{eqn:bayes_update})
		\Else
		\State Query direct labels for $\feature_t$ and obtain $\directval_t$
		\State Update $D_t^l\leftarrow D_{t-1}^l\cup \{(x_t,y_t)\}$, $D_t^c\rightarrow D_{t-1}^c$
		\State Perform Bayesian update for $\labsig_{t}, \labsig_t$ based on $D_t^l$ following (\ref{eqn:bayes_update})
		\EndIf        
		\EndProcedure
	\end{algorithmic}
	\label{algo:comp-gp-ucb-twophase}
\end{algorithm}

\subsection{Optimization with Known $\funcdiff$ \label{sec:algo_knowndiff}}

When $\funcdiff$ is known and given, \shiftAlgoName is formally described in Algorithm \ref{algo:comp-gp-ucb-twophase}. Our algorithm works both for unbiased comparisons ($\funcdiff=0$) and misspecified comparisons ($\funcdiff>0$). \shiftAlgoName is an anytime algorithm, meaning that it does not need to know the total budget $\budget$ before it begins. 
For any input $\funcdiff\geq 0$, the high-level idea is to constrain the search region for $\truthfunc$ using comparisons to the set $\baseset:=\{x: \compprobfunc(x) \geq \compprobfunc^* -  L_2\zeta\}$ where $\compprobfunc^* = \max_x \compprobfunc(x)$. $\baseset$ is guaranteed to contain the optimizer $\truthfunc$ under our assumptions; To see this, let $x^*$ be any optimizer of $\truthfunc$, and we have
\(\compprobfunc^*-\compprobfunc(x^*)\leq \liplink (\compfunc^*-\compfunc(x^*)) \leq \liplink(\truthfunc^*-\truthfunc(x^*)+\funcdiff)=\liplink\funcdiff. \)
The first inequality follows from Assumption \ref{asm:bound_diff} and the second one follows from Assumption \ref{asm:funcdiff}. It is easy to see that $\baseset$ is much smaller than $\featurespace$ if comparisons are mostly correct (i.e., $\funcdiff$ is small); therefore we can explore more efficiently by restricting the search on $\baseset$.


\shiftAlgoName takes as input $\funcdiff$, a parameter $\threscomp$ to control exploration on comparisons, and a confidence level $\delta$. We keep track of posteriors $(\labmu,\labsig), (\rmu, \rsigma)$ for $\truthfunc$ and $\compprobfunc$ respectively, and construct confidence intervals $\labmu_{t-1}(x)\pm \boundlab_t\labsig_{t-1}(x)$, $\rmu_{t-1}(x)\pm \boundr_t\rsigma_{t-1}(x)$.
Since $\compprobfunc$ is unknown, to approximate $\baseset$, the algorithm adopts a two-phase approach: In the first phase (Step \ref{step:firstphase}-\ref{step:firstphaseends}), we optimize $\compprobfunc$ using comparison queries until $\boundr_{t}\rsigma_{t-1}(\feature_t)\leq \threscomp$, i.e., the queried point has confidence of at least $\threscomp$. At the end of the first phase, we compute $\estbestfr$ as a lower bound for $\compprobfunc^*$. 
Next, we start the second phase exploring $\truthfunc$ (Step \ref{step:secondphase}-\ref{step:secondphaseends}). We select the query point $x_t$ based on a filtering $\confr_t(x)\geq 0$;
the filtering approximates the constraint set $\baseset$ by combining the current UCB of $\compprobfunc$ and the LCB $\estbestfr$
from the first phase. 
Then the algorithm optimizes the UCB of $\truthfunc$ under the constraint of $\confr_t(x)\geq 0$. While doing this, we check the UCB of $\compprobfunc$ at the maximizer $x_t$ and if we are not confident about $\compprobfunc(x_t)$, we query a comparison, or otherwise we make a direct query as in GP-UCB. 

The query process is described in the procedure \textsc{Query}. For direct queries, we query $x_t$ directly, and update the posterior of $\truthfunc$ according to (\ref{eqn:bayes_update}); for comparisons, we compare $x_t$ with a random point $x'$, and use the result as feedback to update posterior of $\compprobfunc$. We note that this comparison result is an unbiased estimate of $\compprobfunc(x_t)$.

\textbf{The Two-Phase Approach.} Both phases are critical for the algorithm to succeed. The first phase is important in two ways: Firstly, it helps to get a low regret in the unbiased comparisons setting, and in the initial stages of the algorithm when only comparison queries are used for the biased (misspecified comparison) setting. 
Also, it gives a lower bound $\estbestfr \leq \compprobfunc^*$ of the optimum of $\compprobfunc$ at Step \ref{step:lcb} which will be used to approximate the constraint set $\baseset$. Then we use the second phase to obtain low regret in the biased comparison case.

\textbf{Choice of $\confr_t$.} The choice of $\confr_t$ is critical for the algorithm to succeed. We want that the region $S=\{x: \confr_t(x)\geq 0\}$ is not too small or too large: we need that every maximizer $\feature^*$ of $\truthfunc$ is in $S$, while also excluding as many points as possible using the information from $\compprobfunc$. To achieve the former, we have added $\liplink\funcdiff$ to the confidence interval to account for the difference in $\compfunc$ and $\truthfunc$. To achieve the latter, we need both a good UCB of $\compprobfunc$ and a good LCB of $\compprobfunc^*=\max \compprobfunc(x)$. The good UCB is ensured by the check at Step \ref{step:checkconf_p2}; we only make direct queries when we are confident enough about $\compprobfunc(x_t)$. The good LCB is ensured by the first phase, where we compute $\estbestfr$; without the first phase $\estbestfr$ can be arbitrarily bad and it will lead to suboptimal direct queries. In the proof we show that when $\confr_t(x)\geq 0$ and $\boundr_{t}(x)\rsigma_{t-1}(x)\geq \gamma$, $x$ belongs to an approximation of $\baseset$. So the two constraints combined ensure that we use direct queries to explore $\baseset$.


\ssp
\subsection{Theoretical Analysis\label{sec:theory_analysis}}
\ssp


We now present our theoretical results. We defer full proofs to the appendix due to space constraints. We first analyze the unbiased comparison case. In this case, we have $\funcdiff=0$, and we only need comparisons to achieve low regret. Therefore we run \shiftAlgoName with $\funcdiff=\threscomp=0$; in this case, the algorithm only executes the first phase, and only uses comparisons to optimize $\compprobfunc$. We obtain the following guarantee.
\begin{theorem}\label{thm:twophase_unbiased}
	Suppose Assumption \ref{asm:bound_diff} holds, and $\compfunc=\truthfunc$.
	Let $\boundr_t = 2B+\sqrt{2\left(\infogain{t-1}{\featurespace}+1+\log(1/\delta)\right)}$.
	There exists a constant $C$ dependent on $d,\kernel$ such that \shiftAlgoName with $\funcdiff=\threscomp=0$ has a simple regret bounded by
	\begin{equation}
	\simpleregret(\budget)\leq 
	C\left(B+\sqrt{\left(\infogain{\upperquery}{\featurespace}+\log(1/\delta)\right)}\right)\sqrt{\frac{\infogain{\upperquery}{\featurespace}}{\upperquery}}. \label{eqn:bound_unbiased}
	\end{equation}
\end{theorem}
\newcommand{\lowerquerytmp}{\lowerquery}
\textbf{Remark.} IGP-UCB \citep{chowdhury2017kernelized} in the label-only setting has regret of form
\begin{align}
&\simpleregret_{\text{IGP-UCB}}(\budget)\leq \label{eqn:regret_igp_ucb} \\
&C\left(B+\sqrt{\left(\infogain{\lowerquerytmp}{\featurespace}+\log(1/\delta)\right)}\right)\sqrt{\frac{\infogain{\lowerquerytmp}{\featurespace}}{\lowerquerytmp}},\nonumber
\end{align}
where $\lowerquerytmp=\lfloor \frac{\budget}{\costdirect}\rfloor$. This is the same form as (\ref{eqn:bound_unbiased}), but with $\upperquery$ replaced with $\lowerquerytmp$. Recall that $\upperquery$ is the number of queries when we use all the budget on comparisons, and $\lowerquery$ is the number for using all budget on direct queries. 
In other words, \shiftAlgoName \emph{has the same rate as IGP-UCB as if direct queries are as cheap as comparisons}. When comparisons are much cheaper than direct queries, \shiftAlgoName leads to a great advantage by significantly reducing the number of direct queries needed.

We then analyze \shiftAlgoName in the misspecified comparison setting($\funcdiff>0$). In this setting, comparisons act as a filter on $\featurespace$ to reduce the search region for direct queries. When $\compfunc$ approximates $\truthfunc$ well (i.e., a small $\funcdiff$), the set $\baseset=\{x:\compprobfunc(x)\geq \compprobfunc^*-\liplink\funcdiff \}$ is much smaller than the feature space $\featurespace$.
Therefore by using comparisons, we wish to replace the $\infogain{n_0}{\featurespace}$ term in (\ref{eqn:regret_igp_ucb}) by $\infogain{n_0}{\baseset}$, effectively exploring a smaller region. 
We show that \shiftAlgoName can have a similar behavior by exploring on a slightly larger set
dependent on $\threscomp$, defined as 
$\baset=\{\feature\in \featurespace: \compprobfunc(\feature)\geq \compprobfunc^*-\liplink\funcdiff-4\threscomp \}$. The following theorem characterizes the regret of \shiftAlgoName under the misspecified comparison setting.

\begin{theorem}\label{thm:twophase}
	Suppose Assumptions \ref{asm:bound_diff} and \ref{asm:funcdiff} hold, and $\funcdiff$ is known.
	Let $\boundr_t$ be the same as in Theorem \ref{thm:twophase_unbiased}, and   $\boundlab_t=2B+\sqrt{2\left(\infogain{t-1}{\baset}+1+\log(1/\delta)\right)}$.
	There exists constants $\startupbudget, C$ dependent on $\funcdiff,\threscomp,B,d, \kernel$ such that if when $\budget\geq \startupbudget$
	we have $\simpleregret(\budget)\leq \min\{S_1, S_2\}$, where
	\begin{align*}
		S_1=&\;2\liplinklower\threscomp+\funcdiff+\nonumber\\
		&\;C\left(B+\sqrt{\left(\infogain{\upperquery}{\featurespace}+\log(1/\delta)\right)}\right)\sqrt{\frac{\infogain{\upperquery}{\featurespace}}{\upperquery}},\\
	S_2=&\;C\left(B+\sqrt{\left(\infogain{\lowerquery}{\baset}+\log(1/\delta)\right)}\right)\sqrt{\frac{\infogain{\lowerquery}{\baset}}{\lowerquery}}.
	\end{align*}
\end{theorem}

We discuss about the bounds and setup of parameters before coming to the proof of Theorem \ref{thm:twophase}.

\textbf{Remarks.} The regret bound in Theorem \ref{thm:twophase} enjoys best of both worlds from comparisons and direct queries. The first bound has the same form as in Theorem \ref{thm:twophase_unbiased} but with another $2\liplinklower\threscomp+\funcdiff$ term. This comes from the first phase of \shiftAlgoName, and the extra term comes from the fact that $\compfunc\ne \truthfunc$. In the second phase, the algorithm achieves the second bound $S_2$, which is the rate of using $\lowerquery$ direct queries to explore $\baset$. Compared with (\ref{eqn:regret_igp_ucb}), the second bound has the same rate on $\lowerquery$, but with a reduced search region $\baset$ and a startup budget $\startupbudget$ for comparisons to work.
When $\compfunc$ is a good approximation to $\truthfunc$, $\baset$ is much smaller than $\featurespace$ and will lead to a great improvement in the number of direct queries needed.
\\
 
\textbf{Setup of parameters.} 1. Setting $\threscomp$: $\threscomp$ acts as a threshold for exploring comparisons in both phases of \shiftAlgoName. A small $\threscomp$ will lead to a small $\baset$ and therefore better regret rates; but it will also make the algorithm spend more time on comparisons before moving to direct queries, i.e., a large $\startupbudget$. One plausible choice for $\threscomp$ is to set $\threscomp=\frac{1}{\liplink}\funcdiff$, and this will make $\baset\approx\baseset$.\\
2. Setting $\boundlab_{t}$: The setup for $\boundlab_{t}$ in Theorem \ref{thm:twophase} requires knowing $\infogain{t}{\baset}$ before algorithm starts and this is unrealistic to set up. However, in practice the default choice for $\boundlab_{t}$ is often very loose and hand-tuned values are used instead (e.g., Kandasamy et. al\citep{kandasamy2016multi} uses $\boundlab_{t}=0.2d\log(2t)$). In this sense this setup for $\boundlab_{t}$ does not affect its practical use.
For theoretical purposes, we can also set $\boundlab_t=\boundr_t$; this leads to a regret rate of $\logO\left(\frac{(B+\sqrt{\infogain{\lowerquery}{\featurespace}})\sqrt{\infogain{\lowerquery}{\baset}}}{\sqrt{\lowerquery}}\right)$, slightly larger than the current rate but still smaller than GP-UCB.

\begin{proof}[Proof Sketch]
	We prove Theorem \ref{thm:twophase} and Theorem \ref{thm:twophase_unbiased} follows as a corollary. For the first bound, if we have left phase 1 and entered phase 2, let $T_0$ be the time that we leave phase 1. By routine calculation we can show
	\begin{align}
	\simpleregret(\budget)\leq \truthfunc^*-\truthfunc(x_{T_0})\leq L_1(\compprobfunc^*-\compprobfunc(x_{T_0}))\leq 2L_1\threscomp+\funcdiff. \label{eqn:regret_phase1_exit}
	\end{align}
	On the other hand, if we do not finish phase 1 (e.g., when $\funcdiff=\threscomp=0$), we can follow the proof of IGP-UCB \citep{chowdhury2017kernelized} and show that 
	\begin{align}
	\simpleregret(\budget)\leq C\boundr_{\upperquery}\sqrt{\frac{\infogain{\upperquery}{\baser}}{\upperquery}}+\funcdiff.\label{eqn:regret_phase1_stay}
	\end{align}
	Combining (\ref{eqn:regret_phase1_exit}) and (\ref{eqn:regret_phase1_stay}) we get the first bound $S_1$. 
	
	Now we show the second bound $S_2$. Suppose the algorithm makes $n$ queries. For any set $A\subseteq \featurespace$, let $\nqueryr_n(A)$ be the number of comparison queries into $A$ when the algorithm has made $n$ queries, and $\nqueryl_n(A)$ be the number of direct queries. We have
	\[n=\nqueryr_n(\featurespace)+\nqueryl_n(\basec)+\nqueryl_n(\baset). \]
	For the first term, we show that there exists a constant $C_\kernel$ such that $\nqueryr_n(\featurespace)\leq C_{\kernel}\left(\frac{\boundr_{n}}{\threscomp}\right)^{p+2}$, where $p=d$ for SE kernel and $p=2d$ for Mat\'{e}rn kernel. For the second term, we show that our algorithm makes sure that $\nqueryl_n(\basec)=0$, i.e., it always query in $\baser$ when it uses direct queries. These two results combined can show that we allocate at least $\lowerquery/2$ direct queries to explore $\baset$. The second bound $S_2$ then follows by bounding the regret similar to IGP-UCB.
\end{proof}

\begin{algorithm}[ht!]
	\caption{\shiftAlgoName for unknown $\funcdiff$}
	\begin{algorithmic}[1]
		\Require Threshold $\threscomp$, comparison bias starting point $\funcdiff_0$, bias upper bound $\diffupper$, budget $\budget$
		\State Set $D_0^r=D_0^l=\emptyset$, $(\rmu_0,\rsigma_0)=(\labmu_0,\labsig_0)=(0,\kernel^{1/2}),t\leftarrow 0, \iterdiff\leftarrow 0, \nlabelq\leftarrow 0$
		\Repeat
		\State Compute $x_t=\argmax_{\feature\in \featurespace} \rmu_{t-1}(\feature)+\boundr_{t}\rsigma_{t-1}(\feature)$
		\State \textsc{Query}($x_t,\comptype$)
		\Until{$\boundr_{t}\rsigma_{t-1}(\feature_t)\leq \threscomp$}
		\State Let $\estbestfr=\rmu_{t-1}(\feature_t)-\boundr_{t}\rsigma_{t-1}(\feature_t)$
		\While{$\funcdiff_{\iterdiff}\leq \diffupper$}
		\State Let $\confr_t(x)=\rmu_{t-1}(x)+\boundr_t\rsigma_{t-1}(x)-\estbestfr+2\liplink\funcdiff_\iterdiff$
		\State Compute $x_t=\argmax_{x\in \featurespace: \confr_t(x)\geq 0} \labmu_{t-1}(x)+\boundlab_t\labsig_{t-1}(x)$
		\If{$\boundr_t(x_t)\rsigma_{t-1}(x_t)\geq \gamma$} \textsc{Query}$(x_t,\text{comp})$ 
		\Else 
		\State \textsc{Query}$(x_t,\text{label})$
		\State $\nlabelq \leftarrow \nlabelq + 1$
		\EndIf
		\If{$\nlabelq\geq \frac{\lowerquery}{2\lceil \log(\diffupper/\funcdiff_0)\rceil}$\label{step:doublediff} }
		\State $\nlabelq\leftarrow 0, \funcdiff_{\iterdiff+1}\leftarrow 2\funcdiff_{\iterdiff}, \iterdiff\leftarrow \iterdiff+1$
		\EndIf
		\State $t\leftarrow t+1$
		\EndWhile
	\end{algorithmic}
	\label{algo:adapt-bias}
\end{algorithm}

\ssp
\subsection{\shiftAlgoName with Unknown $\funcdiff$ \label{sec:adapt_noise}}
\ssp

\newcommand{\initdiff}{\funcdiff_{0}}
In practice, we cannot know $\funcdiff$ in general, and it is even hard to verify whether $\|\compfunc-\truthfunc\|_\infty \leq \funcdiff$ holds for a given $\funcdiff$. On the other hand, we can often know an upper bound $\diffupper$ such that $\funcdiff\leq \diffupper$. For example, if we know both $\truthfunc$ and $\compfunc$ are bounded in $[-B_{\infty}, B_{\infty}]$ 
we naturally have $\|\compfunc-\truthfunc\|_\infty\leq 2B_{\infty}$. However, Algorithm \ref{algo:comp-gp-ucb-twophase} is not useful if we set $\funcdiff=2B_\infty$, because that will lead to a constraint set $\baseset=\{x:\compprobfunc(x)\geq \compprobfunc^*-2\liplink B_\infty \}$ 
that can be as large as $\featurespace$ and we have to explore the whole feature space with direct queries. We develop a slightly different method in Algorithm \ref{algo:adapt-bias} that tries to search $\funcdiff$ between an initial value $\initdiff$ and the upper bound $\diffupper$, and adapts to the true $\funcdiff$. 

Algorithm \ref{algo:adapt-bias} works in the finite-horizon scenario, where the budget $\budget$ is given as input. The process of Algorithm \ref{algo:adapt-bias} is mostly similar to Algorithm \ref{algo:comp-gp-ucb-twophase}, except that it uses $\funcdiff_{\iterdiff}$ in the second phase in place of $\funcdiff$. 
We optimize the function as if Assumption \ref{asm:funcdiff} holds for $\funcdiff_{\iterdiff}$. $\funcdiff_{\iterdiff}$ starts from $\funcdiff_0$; at step \ref{step:doublediff}, once we have spent enough queries at the current estimate of $\funcdiff$, we double the current $\funcdiff_{\iterdiff}$. We iterate until we reach $\funcdiff_{\iterdiff}>\diffupper$.
The threshold for $\nlabelq$, $\frac{\lowerquery}{2\lceil \log(\diffupper/\funcdiff_0)\rceil}$, is chosen such that 
we divide a label budget of $\lowerquery/2$ direct queries equally among the $\lceil \log(\diffupper/\funcdiff_0)\rceil$ possible values of the $\funcdiff_{\iterdiff}$'s. The constant 2 is chosen arbitrarily here; any choice of $\lowerquery/c$ for a constant $c>1$ will obtain the same rate.

We present our theoretical results as a corollary to Theorem \ref{thm:twophase}. Since we cannot find the exact $\funcdiff$, our results depend on a slightly larger $\equivdiff=\max\{2\funcdiff, \funcdiff_0\}$. We use $\basetp=\{\feature\in \featurespace: \compprobfunc(\feature)\geq \compprobfunc^*-2\liplink \equivdiff-4\threscomp \}$ to represent the constraint set of interest when $\funcdiff$ is replaced by $\equivdiff$.

\begin{corollary}\label{col:unknown_diff}
	Suppose Assumptions \ref{asm:bound_diff} and \ref{asm:funcdiff} hold, and $\funcdiff\leq \diffupper$.
	Under the same setting of $\boundr_t,C,\budget_0$ as in Theorem \ref{thm:twophase}, and $\boundlab_t=2B+\sqrt{2\left(\infogain{t-1}{\basetp}+1+\log(1/\delta)\right)}$,
	the simple regret of Algorithm \ref{algo:adapt-bias} satisfies $\simpleregret(\budget)\leq \min\{S_1,S_2\}$, where
	\begin{align*} 
	S_1=&\;2\liplinklower\threscomp+2\funcdiff+\\
	&\;C\left(B+\sqrt{\left(\infogain{\upperquery}{\featurespace}+\log(1/\delta)\right)}\right)\sqrt{\frac{\infogain{\upperquery}{\featurespace}}{\upperquery}},\\
	S_2=&\;C\left(B+\sqrt{\left(\infogain{\lowerquery}{\basetp}+\log(1/\delta)\right)}\right)\sqrt{\frac{\infogain{\lowerquery}{\basetp}}{\lowerquery}}.
	\end{align*}
\end{corollary}
\textbf{Remark.} 1. The regret rate of Corollary \ref{col:unknown_diff} is almost the same as Theorem \ref{thm:twophase}, except the set $\baset$ is replaced with $\basetp$. We note that all the terms in the regret rate depend only on $\equivdiff$ or $\funcdiff$, and do not depend on $\diffupper$.  This means Algorithm \ref{algo:adapt-bias} can adapt to unknown level of comparison bias $\funcdiff$. \\
2. Similar to Theorem \ref{thm:twophase}, Corollary \ref{col:unknown_diff} also requires the unknown quantity $\infogain{t}{\basetp}$ to set $\boundlab_t$; in practice we can also use a similar hyper-parameter search to find this quantity. $\threscomp$ also takes a similar effect as in Algorithm \ref{algo:comp-gp-ucb-twophase}, and $\threscomp=\frac{1}{\liplink}\initdiff$ can lead to $\hat\baset\approx \baseset$ and a practical algorithm. Again, setup of these parameters only depends on $\equivdiff$ and is not affected by $\diffupper$.\\
\begin{figure*}
	\msp
	\centering
	\begin{subfigure}[b]{0.33\textwidth}
		\centering
		\includegraphics[width=\textwidth]{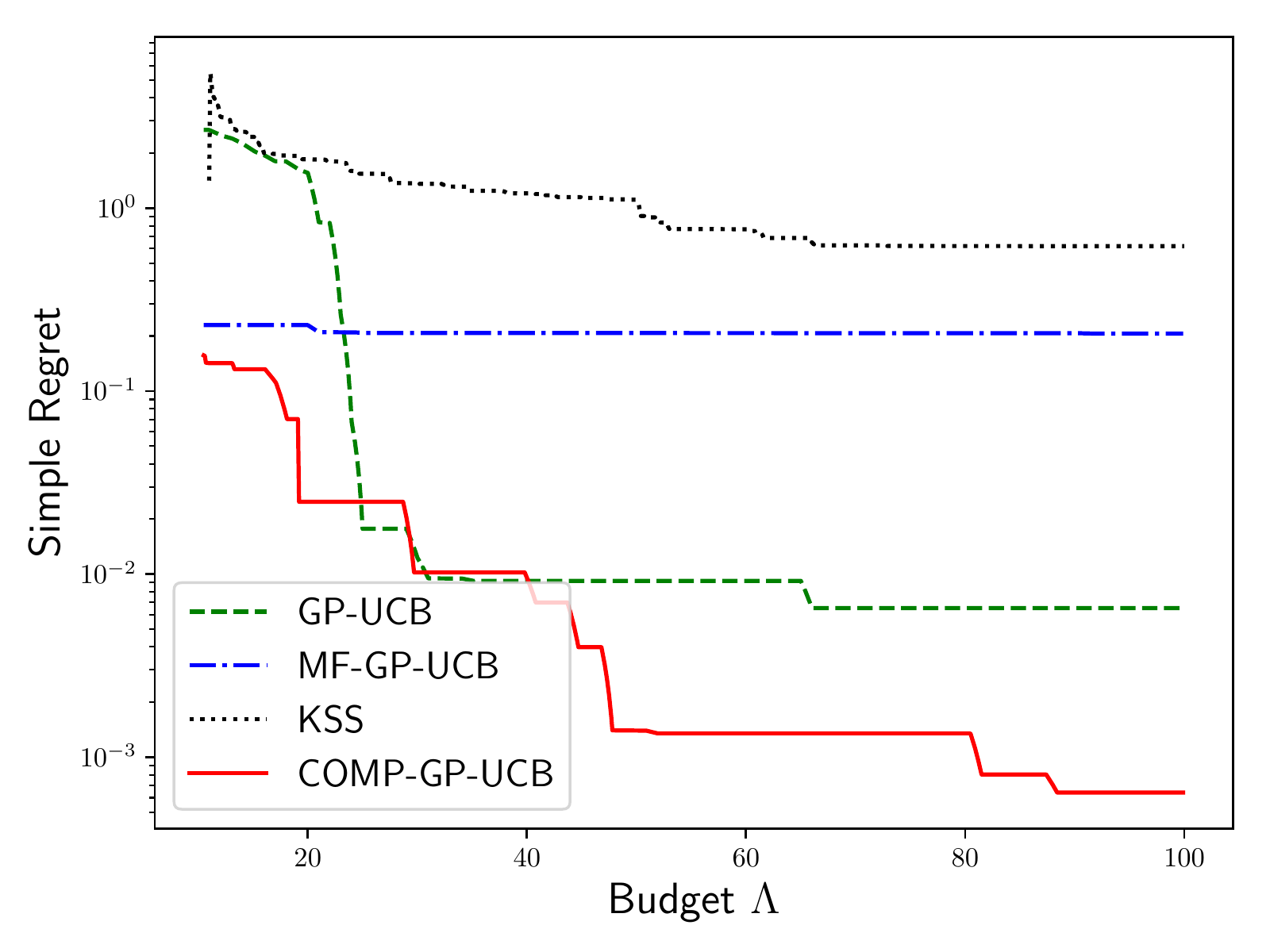}
		\caption{CurrinExp, varying budget \label{fig:currin_budget}}
	\end{subfigure}%
	\hfill
	\begin{subfigure}[b]{0.33\textwidth}
		\centering
		\includegraphics[width=\textwidth]{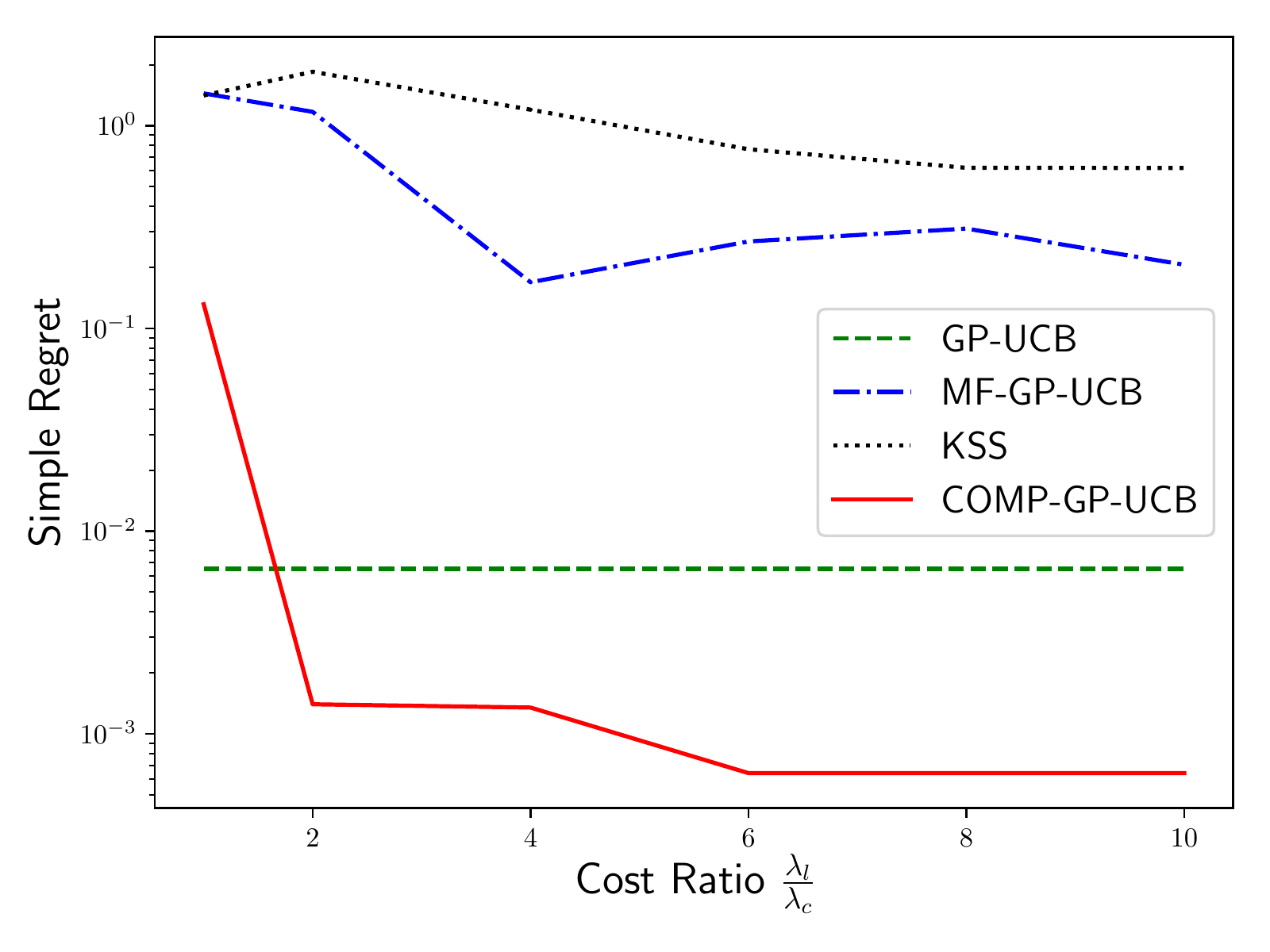}
		\caption{CurrinExp, varying cost ratio\label{fig:currin_ratio}}
	\end{subfigure}%
	\hfill
	\begin{subfigure}[b]{0.33\textwidth}
		\centering
		\includegraphics[width=\textwidth]{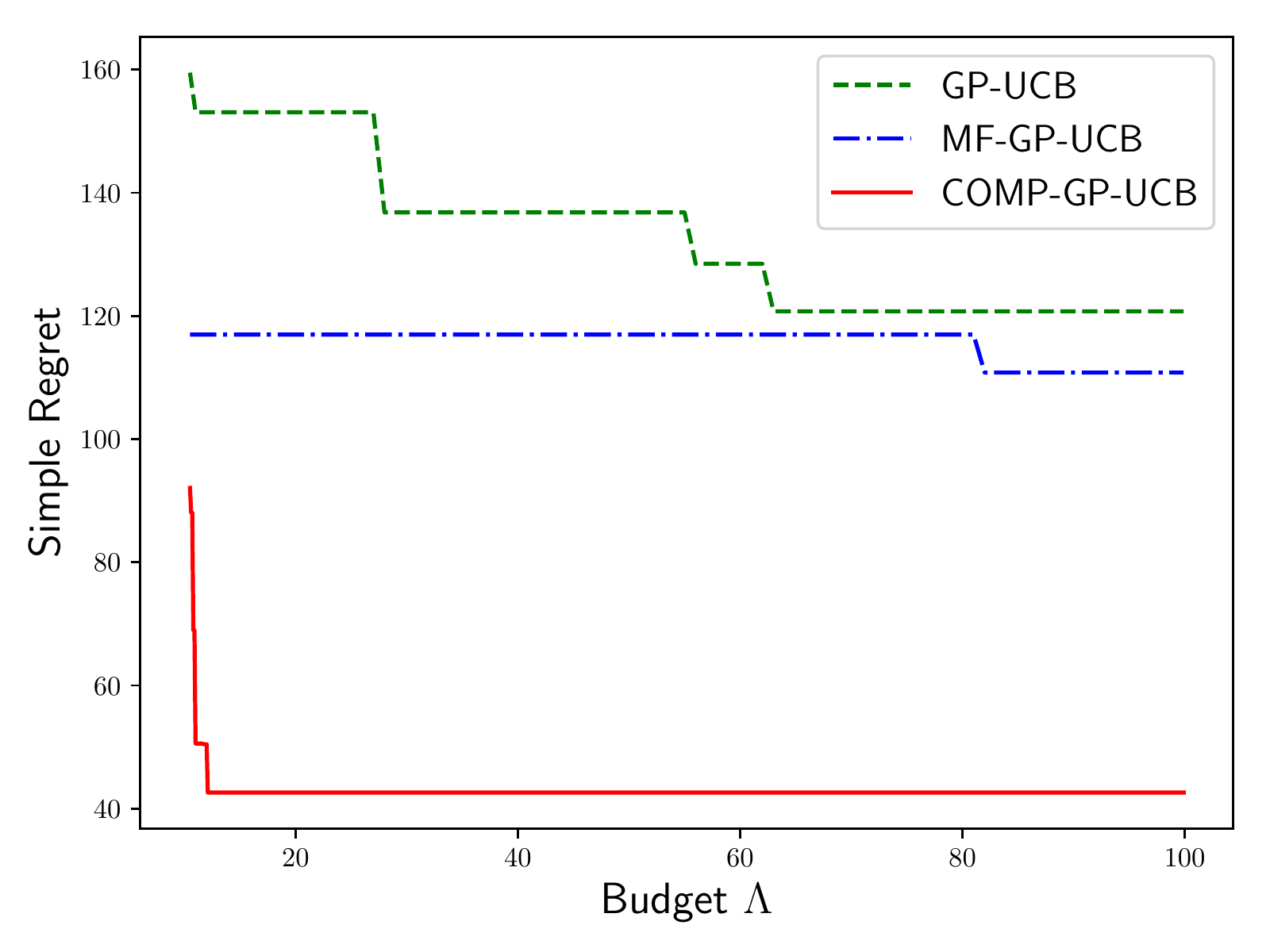}
		\caption{BoreHole, varying budget \label{fig:borehole_budget}}
	\end{subfigure}
	
	\caption{Empirical results comparing \shiftAlgoName with baseline methods. KSS stands for KernelSelfSparring. \label{fig:expr_result}}
	\lsp\msp
\end{figure*}

\msp
\subsection{Comparison with MF-GP-UCB \citep{kandasamy2016multi} \label{sec:with_mfgpucb}}
\ssp

Our setting and method share some common characteristics as the multi-fidelity method MF-GP-UCB \citep{kandasamy2016multi}, and we formally discuss them here. Our setup is similar to MF-GP-UCB in the two-fidelity case, where the algorithm has access to the target function $\truthfunc$ and its approximation $f^{(1)}$, with $\|\truthfunc-f^{(1)}\|_\infty\leq \funcdiff$ for some known $\funcdiff>0$. Although we also assume $\compfunc$ is a good approximation for $\truthfunc$ (in a weaker sense of being close in terms of $\truthfunc^*$ and $\compfunc^*$, see Assumption \ref{asm:funcdiff}), our setting is harder than MF-GP-UCB and their algorithm cannot be directly applied in our case. This is because we cannot directly query $\compfunc$: $\compfunc$ is only available through comparisons, and we will get the same set of comparisons from $\compfunc$ and $\compfunc+c$ for any constant $c$. In our case, we can only get unbiased estimates for $\compprobfunc$. However, it is unlikely that $\|\compprobfunc-\truthfunc\|_\infty$ is small, because $\compprobfunc(x)\in [0,1]$ for all $x$ since it is the probability of beating a random point, whereas $\truthfunc$ can have arbitrary values. 

MF-GP-UCB bears some resemblance to the second phase in our Algorithm \ref{algo:comp-gp-ucb-twophase}, but they are principally different in choosing the next query point $x_t$. In the MF-GP-UCB algorithm, we have access to another function $\truthfunc'$ similar to $\truthfunc$. The algorithm constructs two sets of UCBs $\phi(x), \phi'(x)$ for $\truthfunc$ and $\truthfunc'$ separately, and use $\min\{\phi(x),\phi'(x) \}$ as a final UCB. In our case, UCBs of $\compprobfunc$ and $\truthfunc$ are not comparable. Instead we use a novel constrained optimization approach based on observations in the first phase.

Another difference is that MF-GP-UCB needs the function difference $\funcdiff$ known beforehand, whereas our modified \shiftAlgoName (Algorithm \ref{algo:adapt-bias}) can adapt to an unknown $\funcdiff$. We note that MF-GP-UCB does use a doubling mechanism in their experiments to make it practical, but they do not provide any theoretical guarantees.

\ssp

\msp
\section{Experiments}
\ssp
We perform experiments against plausible baselines to verify our theory and illustrate the efficacy of our algorithm. 

\textbf{Baselines.} We evaluate the performance of COMP-GP-UCB against the following baselines:\\
\textsf{GP-UCB}\citep{srinivas2009gaussian}: The label-only algorithm optimizing UCB of GP posterior. \\
\textsf{KernelSelfSparring}\citep{sui2017multi}: A comparison-only algorithm that uses Thompson Sampling to optimize comparisons. We note that since $\truthfunc\ne \compfunc$, optimizing comparisons cannot lead to the global optimum. \\
\textsf{MF-GP-UCB}\citep{kandasamy2016multi}: Although MF-GP-UCB is not directly applicable in our case, we try to use it by using comparisons as the lower fidelity. When the algorithm selects to query the lower fidelity on $x_t$, we compare $x$ to a random point $X\in \featurespace$ and use the result as feedback, the same process as \shiftAlgoName.

\textbf{Experiment Setup.}  We apply common techniques in Bayesian optimization to set up hyperparameters of each algorithm (we detail the implementation in appendix). The target functions $\truthfunc, \compfunc$ are set to functions from the multi-fidelity literature, in particular Currin exponential (CurrinExp, $d=2$) and Borehole ($d=8$) \citep{DBLP:journals/technometrics/XiongQW13}. We note that $\truthfunc$ and $\compfunc$ have different values and maximizers. All methods use the RBF kernel for GP.
For all methods we compute the simple regret (\ref{eqn:regret_def}) w.r.t. $\truthfunc$ \footnote{We find that KernelSelfSparring is extremely slow for $d=8$ so we only test it for CurrinExp.}. The results are averaged over 20 runs, with a total budget of $\budget=100$.

\textbf{Cost Ratio.} In practice, the relation between costs of labels and comparisons can be complex. We call $\frac{\costdirect}{\costcomp}$ the \emph{cost ratio} between labels and comparisons; the larger the cost ratio, the cheaper the comparisons. Our algorithm generally works for a cost ratio $\frac{\costdirect}{\costcomp}>1$. We test the performance under various cost ratios in our experiment. For a fair comparison with MF-GP-UCB, we also use their setup of $\costcomp=0.1$ and $\costdirect=1$. 

\ssp
\subsection{Results}
\ssp

The results are summarized in Figure \ref{fig:expr_result}. Firstly we compare the performance on CurrinExp by varying the total budget from 10 to 100 (Figure \ref{fig:currin_budget}). \shiftAlgoName shows the best performance over all budget setups. It is worth noting that 
MF-GP-UCB performs worse than label-only GP-UCB in our setting; this is because the target function of MF-GP-UCB in this case essentially optimizing the function $\compprobfunc$, which is bounded in $[0,1]$, resulting in a very large approximation bias. In contrast, \shiftAlgoName is able to use comparisons in an efficient way to reduce the search space for optimization.

Then in Figure \ref{fig:currin_ratio}, we fix the total budget to be $\Lambda=100$ and cost of labels $\costdirect=1$, and vary the cost ratio from 1 to 10 by varying comparison costs. \shiftAlgoName achieves the best performance for all setups except when $\costcomp=\costdirect=1$ and it is worse than the label-only GP-UCB algorithm. This is expected since our algorithm targets to use cheaper comparisons. Our algorithm can be more effective even with a fairly small cost ratio.

Finally, the result on Borehole is depicted in Figure \ref{fig:borehole_budget}. As with CurrinExp, \shiftAlgoName achieves the best performance with a large gap under all budget setups. Due to space limit we put the varying cost ratio result, as well as other experiments in the appendix.

\msp	
\section{Conclusion}
\ssp
We consider a novel dueling-choice setting when both direct queries and comparisons are available for non-convex optimization. We propose the \shiftAlgoName algorithm that 
can achieve benign regret rates in the dueling-choice setting, and can adapt to unknown biases in the comparisons. Our algorithm can also be of independent interest for other multi-fidelity or transfer learning settings where information gleaned from one fidelity or source domain can be actively transferred to optimize the target domain function, under milder conditions than existing literature.

\clearpage
\bibliography{yichongref}
\clearpage
\appendix
\onecolumn

\section{Reviews of the GP-UCB and IGP-UCB Algorithm}

The GP-UCB \citep{srinivas2009gaussian} and IGP-UCB \cite{chowdhury2017kernelized} can be unified as in Algorithm \ref{algo:igp-ucb}. The algorithms only differ at their assumptions and thus the choice of $beta_t$. Our setting of $\boundr_t$ and $\boundlab_{t}$ is similar to IGP-UCB, as we focus more on the agnostic function setting.

\begin{algorithm}[ht!]
	\caption{GP-UCB and IGP-UCB}
	\begin{algorithmic}[1]
		\Require Budget $\budget$
		\State Set $D^l_0=\emptyset$, $(\labmu_0,\labsig_0)=(0,\kernel^{1/2}),t\leftarrow 0$
		\For{$t=1,2,...,\lowerquery$}
		\State Compute $x_t = \argmax_{\feature\in \featurespace} \mu_{t-1}(\feature) + \boundlab_{t}\sigma_{t-1}(\feature)$
		\State Query $f(x_t)$ and obtain feedback $y_t$
		\State Use $y_t$ and (\ref{eqn:bayes_update}) to perform posterior updates, and obtain $\labmu_{t}, \labsig_t$
		\EndFor
	\end{algorithmic}
	\label{algo:igp-ucb}
\end{algorithm}

\section{Experiment Details}
We apply basic techniques in Bayesian Optimization to conduct the experiments.\\
\textbf{Initial queries:} All the algorithms were initialized with uniform random queries with an initial budget of $\budget_0=10$. For multi-fidelity methods (MF-GP-UCB and \shiftAlgoName), we use $\budget_0/2$ on comparisons and $\budget_0/2$ on direct queries; for GP-UCB we use all $\budget_0$ on labels.\\
\textbf{Choice of kernel parameters:} We estimate the kernel bandwidth and scale by maximizing marginal likelihood with respect to the initial random queries. We also update the kernel parameters by maximizing the marginal likelihood for the GP over the lower fidelity function after every 20 iterations, and for the GP over the true function after every 5 iterations. \\
\textbf{Setup of $\boundlab_t$ and $\boundr_t$} We follow MF-GP-UCB and set $\boundlab = 0.5*\log(2*t+\epsilon)=\boundr_t$.\\
\textbf{Choosing query points.} We use the DiRect algorithm \citep{Jones1993} to maximize the marginal likelihood subject to parameter bounds, and to find the next query points.

\subsection{Additional Results on Borehole}
\begin{figure}
	\centering
	\includegraphics[width=0.35\textwidth]{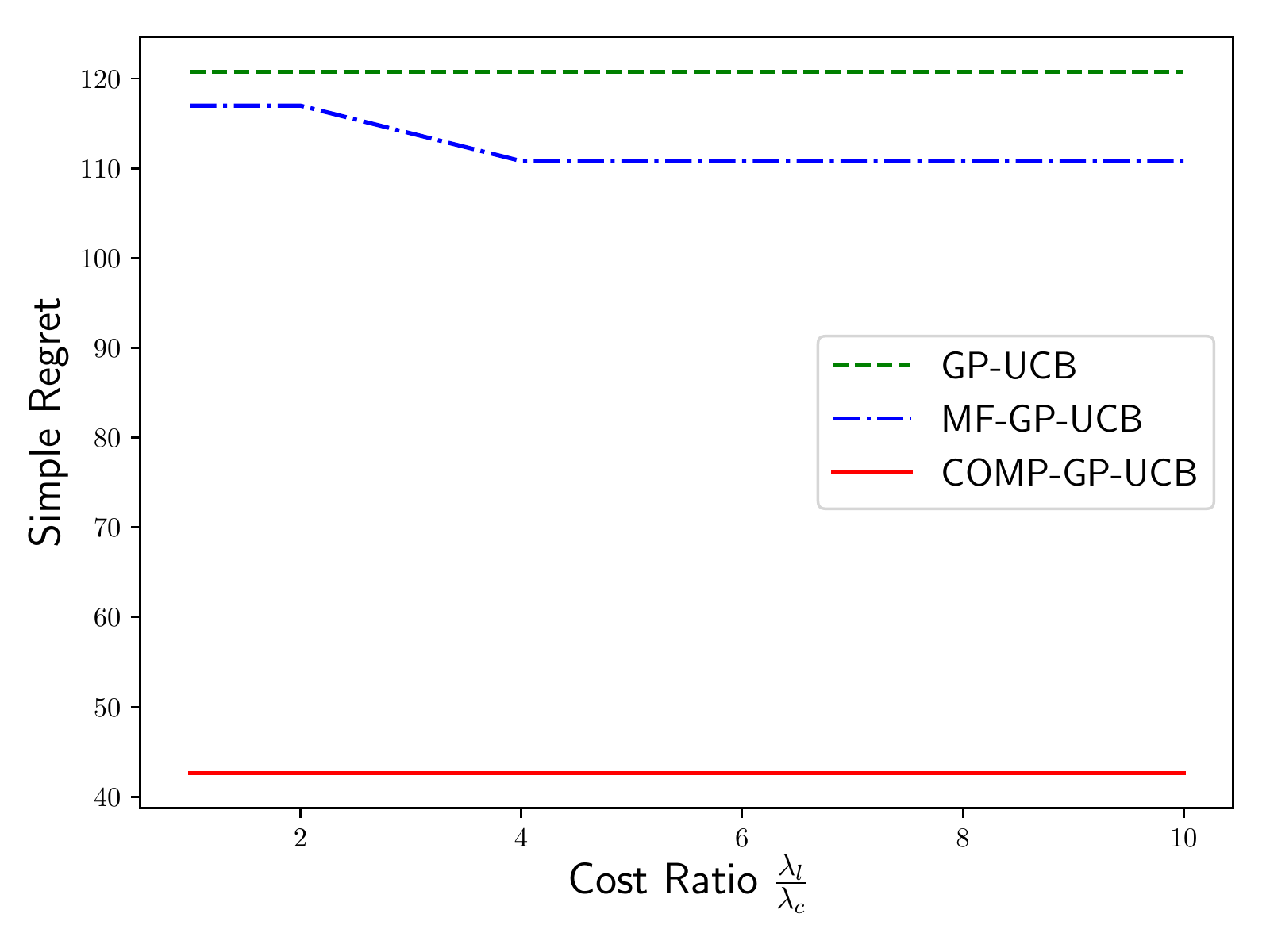}
	\caption{Results on Borehole function with varying cost ratio. \label{fig:borehole_ratio}}
\end{figure}

Figure \ref{fig:borehole_ratio} depicts the results on \textbf{Borehole} function with varying cost ratio. Different than CurrinExp, \shiftAlgoName displays an advantage over the baselines in all cost ratios including $\costdirect=\costcomp$. We suspect this can be because the lower fidelity is easier to explore than the higher fidelity for the Borehole function.

\subsection{Results with Single Fidelity and Same Cost}
Although \shiftAlgoName is designed for the case where comparisons are cheaper than direct queries, it is also interesting to see how it performs when comparisons and direct queries cost the same. We conduct an experiment with $\truthfunc=\compfunc$ and $\costcomp=\costdirect=1$ in Figure \ref{fig:expr_singlefid}. Note that MF-GP-UCB is not applicable in this setting since it is for multi-fidelity setting. \shiftAlgoName performs on par (for CurrinExp) or better (for Borehole) than the baselines in our result. Note that while our theory suggests that the convergence rate of \shiftAlgoName is the same as GP-UCB when $\costcomp=\costdirect$ and $\truthfunc=\compfunc$, in practice the underlying function $\compprobfunc$ (of \shiftAlgoName) might be easier to optimize than $\truthfunc$, because it is bounded and can be smoother than $\truthfunc$. We note that cost ratio $\frac{\costdirect}{\costcomp}=2$ is enough for \shiftAlgoName to surpass the performance of GP-UCB on CurrinExp (see Figure \ref{fig:currin_ratio}). 

\begin{figure*}
	\msp
	\centering
	\begin{subfigure}[b]{0.33\textwidth}
		\centering
		\includegraphics[width=\textwidth]{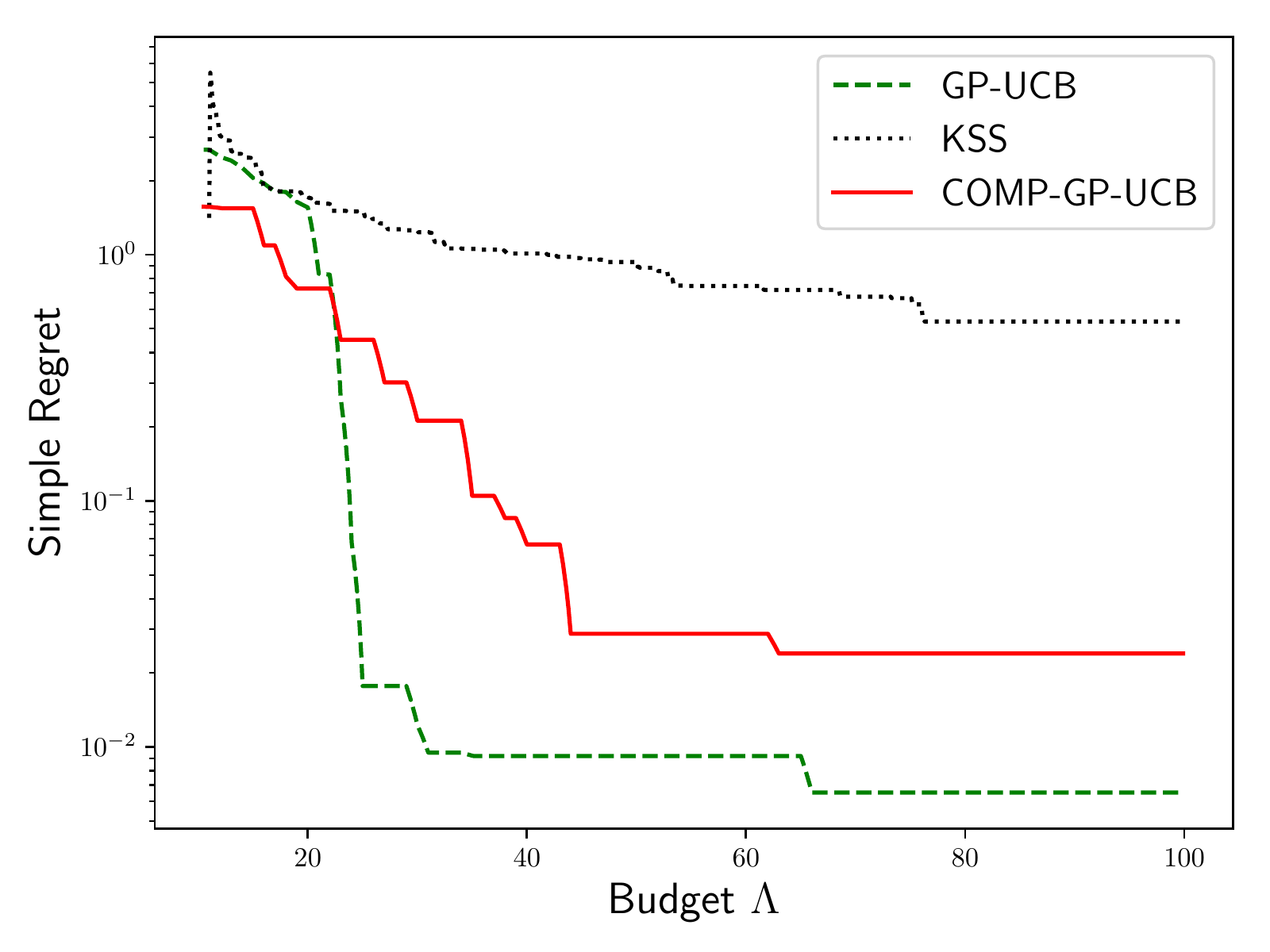}
		\caption{CurrinExp\label{fig:currin_singlefid}}
	\end{subfigure}%
	\begin{subfigure}[b]{0.33\textwidth}
		\centering
		\includegraphics[width=\textwidth]{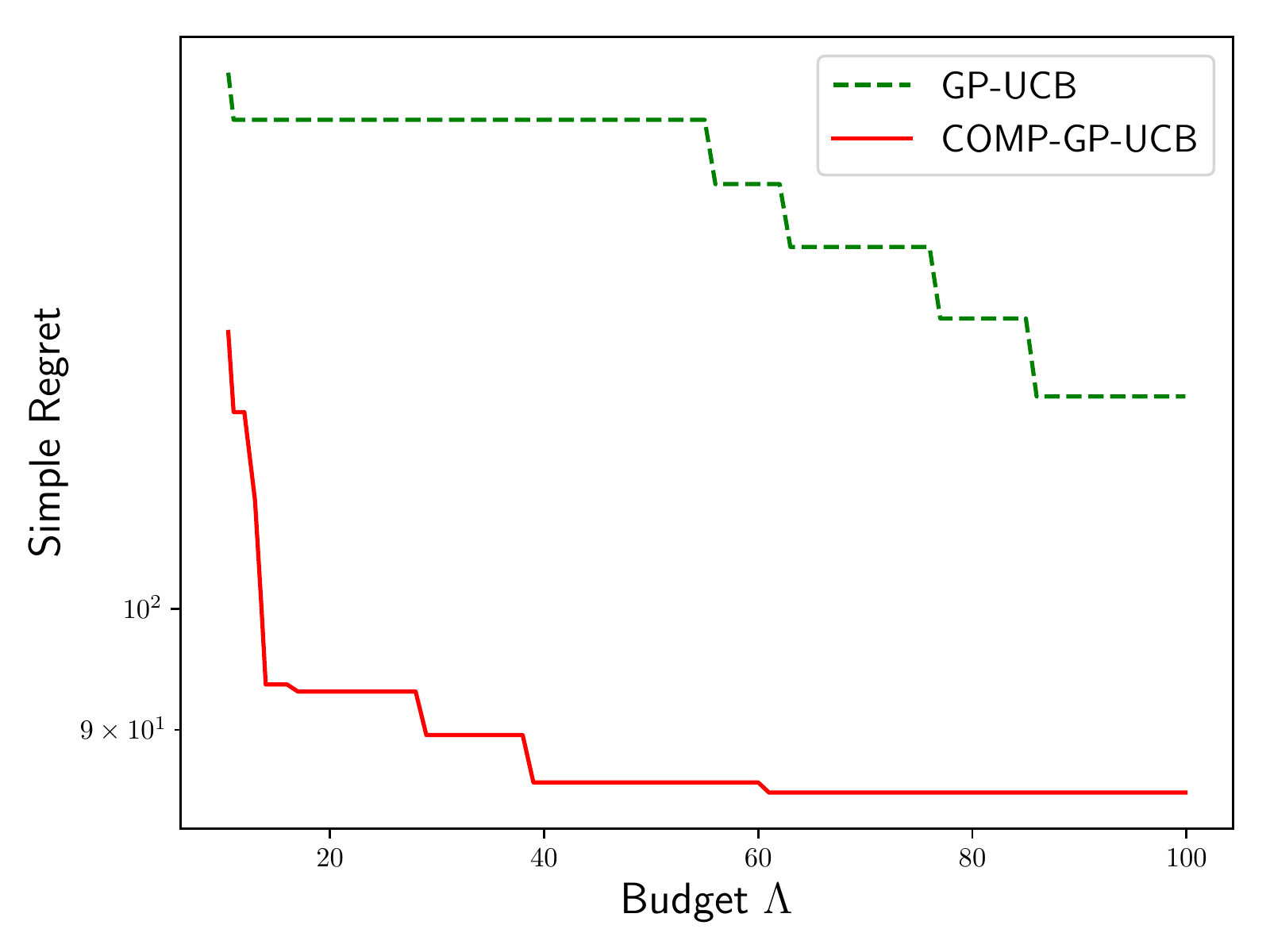}
		\caption{Borehole\label{fig:borehole_singlefid}}
	\end{subfigure}%
	\caption{Empirical results comparing \shiftAlgoName with baselines, under a single fidelity and $\costcomp=\costdirect=1$. KSS stands for KernelSelfSparring. \label{fig:expr_singlefid}}
\end{figure*}

\section{Proofs}
\subsection{Proof of Proposition \ref{prop:liplink}}
	Let $x_c^*$ be a maximizer of $\compfunc$. Because of the link function assumption, $x_c^*$ is also a optimizer of $\compprobfunc$. We have
	\begin{align*}
	\compprobfunc^*-\compprobfunc(x)&=\mathbb{E}[\linkfunc(\compfunc(x_c^*)-\compfunc(X))]-\mathbb{E}[\linkfunc(\compfunc(x)-\compfunc(X))]\\
	&=\mathbb{E}[\linkfunc(\compfunc(x_c^*)-\compfunc(X))-\linkfunc(\compfunc(x)-\compfunc(X))]\\
	&\leq\mathbb{E}[\liplink|\compfunc(x_c^*)-\compfunc(x)|]=\liplink(\compfunc^*-\compfunc(x)).\\
	\end{align*}    
	The lower bound can be proved similarly.

	\subsection{Proof of Theorem \ref{thm:twophase_unbiased} and \ref{thm:twophase}}
	We show Theorem \ref{thm:twophase} and Theorem \ref{thm:twophase_unbiased} follows as a direct corollary. 
	We first use results in IGP-UCB\citep{chowdhury2017kernelized} to obtain confidence bands of $\compprobfunc$:
	\begin{lemma}[Theorem 2, \citep{chowdhury2017kernelized}]\label{lem:confidence_imp2}
		Define $\boundr_t = 2\|\compprobfunc\|_\kernel+\sqrt{2\left(\infogain{t-1}{\featurespace}+1+\log(2/\delta)\right)}$. Then with probability $1-\delta/2$ we have for all time $t$ and any point $\feature\in \featurespace$,
		\[|\rmu_{t-1}(x)-\compprobfunc(x)|\leq \boundr_t\rsigma_{t-1}(x).  \]
	\end{lemma}
	
	Lemma \ref{lem:confidence_imp2} also applies to $\truthfunc, \feature\in \baset, \labmu,\labsig$ by setting $\boundlab_t=2\|\truthfunc\|_\kernel+\sqrt{2\left(\infogain{t-1}{\baset}+1+\log(1/\delta)\right)}$. 
	
	We also use the following lemma to bound the sum of posterior variances:
	\begin{lemma}[Lemma 8, \citep{kandasamy2016multi}]\label{lem:bound_sigmasum}
		Let $A\subseteq \featurespace$. Suppose we have n queries $(\feature_t)_{t=1}^n$ of which $s$ points are in $A$. Then the posterior $\sigma_t$ satisfies
		\[\sum_{x_t\in A} \sigma_{t-1}^2(x_t)\leq \frac{2}{\log(1+\eta^{-2})} \infogain{s}{A}. \]
	\end{lemma}
	
	Suppose the event in Lemma \ref{lem:confidence_imp2} holds for $\truthfunc$ and $\compprobfunc$. We first prove the first bound by looking at comparison queries.
	Firstly, in the first phase when we compute $x_t$ in step $\ref{step:firstphasext}$ we have
	\begin{align}
	\compprobfunc^*-\compprobfunc(x_t)&\leq \rmu_{t-1}(x^*_r)+\boundr_t\rsigma_{t-1}(x^*_r)-(\rmu_{t-1}(x_t)-\boundr_t\rsigma_{t-1}(x_t))\nonumber\\
	&\leq \rmu_{t-1}(x_t)+\boundr_t\rsigma_{t-1}(x_t)-(\rmu_{t-1}(x_t)-\boundr_t\rsigma_{t-1}(x_t))\nonumber\\
	&= 2\boundr_t\rsigma_{t-1}(x_t). \label{eqn:reg_sigma_}
	\end{align}
	The first inequality uses Lemma \ref{lem:confidence_imp2}, and the second inequality is from that $x_t$ is the maximizer of $\rmu_{t-1}(x)+\boundr_t\rsigma_{t-1}(x)$.
	
	Suppose we finish phase 1 and enter phase 2. Let $T_0$ be the time we leave phase 1, then we must have $\boundr_{T_0}\rsigma_{T_0-1}(x_{T_0})\leq \threscomp$. So
	\begin{align}
	S(\Lambda)&\leq \truthfunc^*-\truthfunc(x_{T_0})\nonumber\\
	&\leq \compfunc(x^*)-\compfunc(x_{T_0})+\funcdiff\nonumber\\
	&\leq \compfunc^*-\compfunc(x_{T_0})+\funcdiff\nonumber\\
	&\leq \liplinklower(\compprobfunc^*-\compprobfunc(x_{T_0}))+\funcdiff\leq 2\liplinklower\threscomp+\funcdiff.\label{eqn:p1b1}    
	\end{align}  
	The second inequality is from Assumption \ref{asm:funcdiff}, the fourth inequality is from Assumption \ref{asm:bound_diff}, and the last inequality is from (\ref{eqn:reg_sigma_}).
	
	If we do not finish phase 1, then the number of comparison queries is $N_1\geq \upperquery-1$ and $N_1\leq \upperquery$, and we have
	\begin{align}
	\left(\sum_{t} \left(\compprobfunc^*-\compprobfunc(x_t)\right)\right)^2&\leq N_1\sum_{t:\querytype_t=\text{comparison}} \left(f_r^*-f_r(x_t)\right)^2\nonumber\\
	&\leq N_1\sum_{t:\querytype_t=\text{comparison}} 4\left(\boundr_t\rsigma_{t-1}(x_t)\right)^2\nonumber\\
	&\leq C_1 N_1\left(\boundr_{N_1}\right)^2 \infogain{N_1}{\baset}.\label{eqn:bound_fr}
	\end{align}
	Here $C_1=\frac{8}{\log(1+\labelerrbound^{-2})}$.
	The first step is from the Cauchy-Schwarz inequality, and the second step is from (\ref{eqn:reg_sigma_}); the last step uses Lemma \ref{lem:bound_sigmasum}. 
	
	So
	\begin{align}
	\simpleregret(\Lambda)&\leq \frac{1}{N_1}\sum_{t} \left(f^*-f(x_t)\right)\nonumber\\
	&\leq \frac{2\liplinklower}{N_1}\sum_{t} \left(\compprobfunc^*-\compprobfunc(x_t)\right)+\funcdiff\nonumber\\
	&\leq \frac{\boundr_{N_1}}{N_1}\sqrt{C_1 N_1 \infogain{N_1}{\baset}}+\funcdiff\nonumber\\
	&\leq \frac{C\boundr_{\upperquery}}{\upperquery}\sqrt{ \upperquery \infogain{\upperquery}{\baset}}+\funcdiff\nonumber\\
	&\leq C\left(B+\sqrt{\left(\infogain{\upperquery}{\baset}+\log(1/\delta)\right)}\right)\sqrt{\frac{\infogain{\upperquery}{\baset}}{\upperquery}}+\funcdiff.\label{eqn:p1b2}
	\end{align}
	The second inequality follows from the same process as in (\ref{eqn:p1b1}); the third inequality follows from (\ref{eqn:bound_fr}); the fourth inequality is from $\upperquery \geq N_1\geq \upperquery-1$. Here $C$ is a constant whose value may change from line to line.
	Combining (\ref{eqn:p1b1}) and (\ref{eqn:p1b2}) we get the first bound.
	
	To show the second bound, we examine the regret from direct queries. 
	We first show that $x^*$ is never excluded from our feasible region:
	\begin{claim}\label{claim:xstar_in}
		$\confr_t(x^*)\geq 0$ for all $t$.     
	\end{claim}
	
	\begin{proof}
		Suppose $x^*$ is a maximizer of $\truthfunc$ in $\featurespace$. Then we have
		\begin{align*}
		\confr_t(x^*)&=\rmu_t(x^*)+\boundr_t\rsigma_t(x^*)-\max_{x'} \left\{\rmu_t(x' )-\boundr_t\rsigma_t(x')\right\}+\liplink\funcdiff\\
		&\geq\compprobfunc(x^*)-\compprobfunc^*+\liplink\funcdiff\\
		&=-\liplink(\compfunc^*-\compfunc(x^*))+\liplink\funcdiff\\
		&\geq -\liplink(\truthfunc^*-\truthfunc(x^*)+\funcdiff)+\liplink\funcdiff\\
		&= -\liplink\funcdiff+\liplink\funcdiff=0.
		\end{align*}
		The first inequality is from Lemma \ref{lem:confidence_imp2}; 
		the second inequality is from Assumption \ref{asm:funcdiff}.
	\end{proof}
	
	Let $N$ be the (random) total number of queries under budget $\Lambda$. We know that the support of $N$ lies in $[\lowerquery,\upperquery]$; we now suppose $n$ is any number in $[\lowerquery,\upperquery]$, and prove properties of Algorithm \ref{algo:comp-gp-ucb-twophase} when it uses $n$ queries.
	
	For any set $A\subseteq \featurespace$, let $\nqueryr_n(A)$ be the number of comparison queries into $A$ when the algorithm has made $n$ queries, and $\nqueryl_n(A)$ be the number of direct queries. We have
	\[n=\nqueryr_n(\featurespace)+\nqueryl_n(\basec)+\nqueryl_n(\baset) \]
	since $\basec \cup \baset=\featurespace$. We bound the first two terms using the following two lemmas:
	\begin{lemma}\label{lem:bound_comp} There exists a constant $C_{\kernel}$ dependent on $\kernel,d$ such that
		$\nqueryr_n(\featurespace)\leq C_{\kernel}\left(\frac{\boundr_n}{\threscomp}\right)^{p+2},$
		where $p=d$ for SE kernel and $p=2d$ for Mat\'{e}rn kernel.
	\end{lemma}
	This lemma is proved in Section \ref{sec:proof_lem_bound_comp}.
	\begin{lemma}\label{lem:bound_outbase}
		$\nqueryl_n(\basec)=0$.
	\end{lemma}
	\begin{proof}
		Suppose $\querytype_t=\text{label}$ for some $t$. Then we must have $\confr_t(x_t)\geq 0$, and that $\boundr_t\rsigma_{t-1}(x_t)<\threscomp, \boundr_t\rsigma_{t-1}(x_t^{(r)})<\threscomp$, here $x_t^{(r)}$ is the value of $x_t$ on line 6. Then we have
		\begin{align*}
		\compprobfunc(x_t)&\geq \rmu_{t-1}(x_t)-\boundr_{t}\rsigma_{t-1}(x_t)\\
		&=\confr_t(x_t)-2\boundr_{t}\rsigma_{t-1}(x_t)+\max_{x'} \left\{\rmu_{t-1}(x' )-\boundr_t\rsigma_{t-1}(x')\right\}-\liplink\funcdiff\\
		&\geq \confr_t(x_t)-2\boundr_{t}\rsigma_{t-1}(x_t)+\rmu_{t-1}\left(x_t^{(r)} \right)-\boundr_t\rsigma_{t-1}\left(x_t^{(r)}\right)-\liplink\funcdiff\\
		&\geq 0-2\threscomp + f_r^*-2\threscomp-\liplink\funcdiff\\
		&=f_r^*-4\threscomp-\liplink\funcdiff.
		\end{align*}
		The first inequality is by applying \ref{lem:confidence_imp2}; the second inequality is by letting $x'=x_t^{(r)}$; the third inequality is by noticing that 
		\[\rmu_{t-1}\left(x_t^{(r)} \right)-\boundr_t\rsigma_{t-1}\left(x_t^{(r)}\right)\geq \left[\rmu_{t-1}\left(x_t^{(r)} \right)+\boundr_t\rsigma_{t-1}\left(x_t^{(r)}\right)\right]-2\boundr_t\rsigma_{t-1}\left(x_t^{(r)}\right)\geq f_r^*-2\threscomp.\]
	\end{proof}
	
	Lemma \ref{lem:bound_comp} shows that we will not make too many queries on comparisons, whereas Lemma \ref{lem:bound_outbase} shows that we always query $x_t\in \baset$ when $\querytype_t=\text{label}$.
	Now let $\startupcomp$ be the smallest number such that for any $\budget\geq \startupcomp\costcomp$ we have
	\[\costdirect\left(\frac{2r\sqrt{d}}{\varepsilon_{\upperquery}}\right)^d \left(\frac{2\labelerrbound\boundr_{\upperquery}}{\threscomp}\right)^2=C_{\kernel}\left(\frac{\boundr_{\upperquery}}{\threscomp}\right)^{p+2} \leq \frac{\Lambda}{2}\]
	where $C_\kernel$ and $p$ are defined in (\ref{eqn:cover_number}). Such $\Lambda_0$ is guaranteed to exist, since $\boundr_{\upperquery}$ grows in sublinear rate for linear, SE and Mat\'{e}rn kernels on $\upperquery$, and therefore $\Lambda$. 
	Thus the number of queries to $\compfunc$, $T_n^c(\featurespace)$, is at most $\lowerquery/2$, and therefore we query at least $\lowerquery/2$ times on direct queries, $\nqueryl_N(\baset)\geq \lowerquery/2$. 
	
	We now follow a similar path as for comparison queries to bound the regret based on direct queries.
	Note that if we use direct query at round $t$, we have $x_t\in \baset$ and that
	\begin{align}
	f^*-f(x_t)&\leq \labmu_{t-1}(x^*)+\boundlab_t\labsig_{t-1}(x^*)-(\labmu_{t-1}(x_t)-\boundlab_t\labsig_{t-1}(x_t))\nonumber\\
	&\leq \labmu_{t-1}(x_t)+\boundlab_t\labsig_{t-1}(x_t)-(\labmu_{t-1}(x_t)-\boundlab_t\labsig_{t-1}(x_t))\nonumber\\
	&= 2\boundlab_t\labsig_{t-1}(x_t). \label{eqn:reg_sigma}
	\end{align}
	The first inequality uses Lemma \ref{lem:confidence_imp2}, and the second inequality uses Claim \ref{claim:xstar_in} and that $x_t$ is the maximizer of $\labmu_{t-1}(x)+\boundlab_t\labsig_{t-1}(x)$.
	
	Now therefore
	\begin{align*}
	\left(\sum_{t:\querytype_t=\labeltype} \left(f^*-f(x_t)\right)\right)^2&\leq T_n^t\left(\baset\right)\sum_{t:\querytype_t=\labeltype} \left(f^*-f(x_t)\right)^2\\
	&\leq T_n^t\left(\baset\right)\sum_{t:\querytype_t=\labeltype} 4\left(\boundlab_t\labsig_{t-1}(x_t)\right)^2\\
	&\leq C_1 T_n^t\left(\baset\right)(\boundlab_n)^2 \infogain{T_n^t\left(\baset\right)}{\baset}
	\end{align*}
	Here $C_1=\frac{8}{\log(1+\labelerrbound^{-2})}$.
	The first step is from the Cauchy-Schwarz inequality, and that $\nqueryl_n(\baser)=0$; the second step is from (\ref{eqn:reg_sigma}); the last step uses Lemma \ref{lem:bound_sigmasum}. 

	So
	\begin{align*}
	\simpleregret(\Lambda)&\leq \frac{1}{T_N^t(\baset)}\sum_{t:x_t\in \baset, \querytype_t=\labeltype} \left(f^*-f(x_t)\right)\\
	&\leq \frac{\boundlab_N}{T_N^t(\baset)}\sqrt{C_1 T_N^t\left(\baset\right) \infogain{T_N^t\left(\baset\right)}{\baset}}\\
	&\leq \frac{C\boundlab_{\lowerquery}}{\lowerquery}\sqrt{ \lowerquery \infogain{\lowerquery\left(\baset\right)}{\baset}}\\
	&\leq C\left(B+\sqrt{\left(\infogain{\lowerquery}{\baset}+\log(1/\delta)\right)}\right)\sqrt{\frac{\infogain{\lowerquery}{\baset}}{\lowerquery}}.
	\end{align*}
	
	\subsection{Proof of Lemma \ref{lem:bound_comp} \label{sec:proof_lem_bound_comp}}
	We use the following lemma from \citep{kandasamy2016multi}\footnote{The original lemma from \citep{kandasamy2016multi} assumes $\truthfunc \sim GP(0,\kernel)$ and $\varepsilon\sim \mathcal{N}(0,\labelerrbound^2)$, but exactly the same proof applies without these assumptions. \label{footnote:sameproof}}:
	
	\begin{lemma}[Lemma 13, \citep{kandasamy2016multi}] \label{lem:boundvar}
		Let
		$A \subseteq \featurespace$ such that its L2 diameter diam$(A)\leq  D$. Say we have n queries $(\feature_t)_{t=1}^n$ of which $s$ points are in $A$. Then the posterior variance of the GP, $\kernel'(\feature, \feature)$ at any $\feature \in A$ satisfies
		\[\kernel'(\feature, \feature)\leq \begin{cases}
		C_{SE}D^2+\frac{\labelerrbound^2}{s}, & \text{ if }\kernel \text{ is the SE kernel},\\
		C_{Mat}D+\frac{\labelerrbound^2}{s}, & \text{ if }\kernel \text{ is the Mat\'{e}rn kernel},\\    
		\end{cases} \]
		for appropriate kernel dependent constants $C_{SE}, C_{Mat}$.
	\end{lemma}
	\newcommand{\fset}{\mathcal{F}}
	
	Consider the SE kernel and the comparison oracle, and a $\varepsilon_n=\frac{\threscomp}{\boundr_n\sqrt{8C_{SE}}}$ covering $(B_i)_{i=1}^n$ of $\featurespace$. We claim that the number of comparison queries inside any $B_i$ is at most $\lceil 2(\frac{\labelerrbound\boundr_n}{\threscomp})^2 \rceil$: suppose we have already queried $\lceil 2(\frac{\labelerrbound\boundr_n}{\threscomp})^2 \rceil$ samples in $B_i$ at some time $t<n$. By Lemma \ref{lem:boundvar} we have
	\[\max_{x\in B_i} \kappa^{(r)}_{t-1}(x,x)\leq C_{SE}(2\varepsilon_n)^2+\frac{\labelerrbound^2}{2(\frac{\labelerrbound\boundr_n}{\threscomp})^2}\leq \left(\frac{\threscomp}{\boundr_n}\right)^2. \]
	Therefore $\boundr_n\rsigma_{t-1}(x)\leq \boundr_{t}\rsigma_{t-1}(x)\leq \threscomp$. Note that whenever $\querytype_t=\text{comp}$, we always have $\boundr_{t}\rsigma_{t-1}(x_t)\leq \threscomp$; so the event that $\querytype_t=\text{comp}$ and $x_t\in B_i$ will not happen until time $n$. We can obtain a similar result for Mat\'{e}rn kernel with $\varepsilon_n=\frac{\threscomp^2}{4C_{Mat}(\boundr_n)^2}$. Therefore we have
	\begin{align}
	\nqueryr_n(\featurespace)\leq \Omega_{\varepsilon_n}(\featurespace)\lceil 2(\frac{\labelerrbound\boundr_n}{\threscomp})^2 \rceil\leq \left(\frac{2r\sqrt{d}}{\varepsilon_n}\right)^d \left(\frac{2\labelerrbound\boundr_n}{\threscomp}\right)^2=C_{\kernel}\left(\frac{\boundr_n}{\threscomp}\right)^{p+2}. \label{eqn:cover_number}
	\end{align}
	Here $\Omega_{\varepsilon_n}(\featurespace)$ is the covering number of $\featurespace$, and we bound the covering number as $\Omega_{\varepsilon_n}(\featurespace)\leq \left(\frac{2r\sqrt{d}}{\varepsilon_n}\right)^d$. Here $C_\kernel=2^{2.5d+2}r^dd^{d/2}C_{SE}^{d/2}\labelerrbound^2$ and $p=d$ for SE kernel, while $C_{\kernel}=2^{3d+2}r^dd^{d/2}C_{Mat}^d\labelerrbound^2$ and $p=2d$ for Mat\'{e}rn kernel.
	
	\subsection{Proof of Corollary \ref{col:unknown_diff}}
	\begin{proof}
		
		Firstly, for the first regret bound, we have the same guarantee as in Theorem \ref{thm:twophase} since the first phase is exactly the same. For the second bound, when $\budget\geq \budget_0$, we allocate at least budget of $\budget/2$ on direct queries. Since we double $\funcdiff_\iterdiff$ in each iteration, at some iteration $\iterdiff_0=O(\log(\equivdiff/\funcdiff_0))$ we will have $\funcdiff_{\iterdiff_0}\in [\funcdiff,\equivdiff]$. Let $\changepoint= \frac{\lowerquery}{2\log(\diffupper/\funcdiff_0)}$. From the proof of Theorem \ref{thm:twophase}, we have the regret $\simpleregret(\Lambda)\leq \boundlab_{\changepoint}\sqrt{\frac{8\infogain{\changepoint}{\basetp}}{\log (1+\eta^{-2})\changepoint}}$ in iteration $\iterdiff_0$. The theorem then follows by realizing $\boundlab_n, \infogain{n}{\basetp}$ grows sublinearly in $n$.
		
	\end{proof}

\end{document}